\documentclass[article,11pt]{article}

\usepackage[latin1]{inputenc}
\usepackage[english]{babel}
\usepackage[T1]{fontenc}

\usepackage{fancyhdr}
\usepackage{wallpaper}
\usepackage{comment}
\usepackage{perpage}
\MakePerPage{footnote} 
\usepackage{pdfpages} 
\usepackage{appendix}

\usepackage{amsthm}
\usepackage{amssymb}
\usepackage{amsmath}
\usepackage{amsfonts}
\usepackage{mathrsfs}
\usepackage[g]{esvect}
\usepackage[mathscr]{eucal}
\usepackage{stmaryrd}
\usepackage{amscd} 
\usepackage{enumitem} 
\usepackage{textcomp}
\usepackage{mathrsfs}
\usepackage{natbib}
\usepackage{latexsym}
\usepackage{mathtools}
\usepackage[ruled]{algorithm2e} 
\allowdisplaybreaks

\usepackage{graphicx}
\usepackage{epsfig} 
\usepackage{pict2e}
\usepackage{tikz}
\usetikzlibrary{patterns}
\usepackage[justification=centering]{subfig}

\usepackage{tabularx}
\usepackage{caption}
\usepackage{array,multirow,multicol,makecell}
\usepackage{geometry}
\usepackage{xspace} 
\usepackage{dsfont}
\usepackage{url}

\newtheorem{prop}{Proposition}
\newtheorem{corr}{Corollary}
\newtheorem{lem}{Lemma}

\newtheorem{theo}{Theorem}
\newtheorem{defin}{Definition}

\theoremstyle{definition}

\newtheorem{rem}{Remark}
\newtheorem{ex}{Example}

\newcommand{\R}{\mathbb{R}}
\newcommand{\RR}{{\rm I}\kern-0.18em{\rm R}}
\newcommand{\h}{{\rm I}\kern-0.18em{\rm H}}
\newcommand{\K}{{\rm I}\kern-0.18em{\rm K}}

\newcommand{\indic}{\mathds{1}} 

\newcommand{\boldx}{\bold x} 
\newcommand{\boldN}{\bold N} 
\newcommand{\boldP}{\bold P} 
\newcommand{\boldR}{\bold R} 
\newcommand{\boldX}{\bold X} 
\newcommand{\boldZ}{\bold Z} 
\newcommand{\boldV}{\bold V} 
\newcommand{\bolda}{\bold a} 
\newcommand{\boldb}{\bold b} 
\newcommand{\boldw}{\bold w} 
\newcommand{\boldu}{\bold u} 
\newcommand{\boldv}{\bold v} 
\newcommand{\boldmu}{\boldsymbol \mu} 
\newcommand{\bolde}{\bold e} 
\newcommand{\boldo}{\bold 0} 
\newcommand{\boldTheta}{\bold \Theta} 
\newcommand{\boldtheta}{\boldsymbol \theta} 
\newcommand{\boldSigma}{\bold \Sigma} 
\DeclareMathOperator{\RV}{RV} 
\DeclareMathOperator{\SRV}{SRV} 
\DeclareMathOperator{\Vect}{Vect} 
\DeclareMathOperator{\Diag}{Diag} 

\renewcommand{\P}{\mathbb{P}} 
\newcommand{\E}{\mathbb{E}} 

\numberwithin{equation}{section}

\title{Sparse regular variation}

\usepackage{authblk}
\author[1]{Nicolas Meyer\footnote{\textit{meyer@math.ku.dk} (corresponding author)}}
\author[1]{Olivier Wintenberger\footnote{\textit{olivier.wintenberger@upmc.fr}}}
\affil[1]{Sorbonne Universit\'e, LPSM, F-75005, Paris, France}

\date{\today}
\begin{document}
\maketitle

\begin{abstract}
Regular variation provides a convenient theoretical framework to study large events. In the multivariate setting, the dependence structure of the positive extremes is characterized by a measure - the spectral measure - defined on the positive orthant of the unit sphere. This measure gathers information on the localization of extreme events and has often a sparse support since severe events do not simultaneously occur in all directions. However, it is defined through weak convergence which does not provide a natural way to capture this sparsity structure.
In this paper, we introduce the notion of sparse regular variation which allows to better learn the dependence structure of extreme events. This concept is based on the Euclidean projection onto the simplex for which efficient algorithms are known. We prove that under mild assumptions sparse regular variation and regular variation are two equivalent notions and we establish several results for sparsely regularly varying random vectors. Finally, we illustrate on numerical examples how this new concept allows one to detect extremal directions.
\end{abstract}


\section{Introduction} \label{sec:introduction}

Estimating the dependence structure of extreme events has proven to be a major issue in many applications. The standard framework in multivariate Extreme Value Theory (EVT) is based on the concept of regularly varying random vectors. Regular variation has first been defined in terms of vague convergence on the compactified space $[-\infty, \infty]^d$ and several characterizations have subsequently been established, see e.g. \cite{resnick_87}, \cite{beirlant_et_al}, \cite{resnick_07}, or \cite{embrechts_kluppelberg_mikosch}. Alternatively, it can also be defined via the convergence of the polar coordinates of a random vector (see \cite{resnick_87}, Proposition 5.17 and Corollary 5.18, or \cite{resnick_07}, Theorem 6.1). Following this approach, a random vector $\boldX \in \R^d_+$ is said to be regularly varying with tail index $\alpha > 0$ and spectral measure $S$ on the positive orthant $\mathbb S^{d-1}_+$ of the unit sphere if
\begin{equation} \label{eq:reg_var_intro}
\P\left( |\boldX|> t x, \boldX / |\boldX| \in B \; \middle| \; |\boldX| > t \right) \to x^{-\alpha} S(B)\, , \quad t \to \infty\, ,
\end{equation}
for all $x > 0$ and for all continuity set $B$ of $S$. This means that the limit of the radial component $|\boldX|/t$ follows a Pareto distribution with parameter $\alpha > 0$ while the angular component $\boldX / |\boldX|$ has limit measure $S$. Moreover, both components of the limit are independent. The measure $S$, called the \emph{spectral measure}, summarizes the tail dependence of the regularly varying random vector $\boldX$. It puts mass in a direction of $\mathbb S^{d-1}_+$ if and only if extreme events appear in this direction. Note that the choice of the norm in \eqref{eq:reg_var_intro} is arbitrary. In this paper, $|\cdot|$ will always denote the $\ell^1$-norm, for reasons explained later.

Based on convergence \eqref{eq:reg_var_intro}, several nonparametric estimation techniques have been proposed to estimate $S$. These approaches tackle nonstandard regular variation for which $\alpha =1$ and all marginals are tail equivalent (possibly after a standardization). Some useful representations of the spectral measure has been introduced in the bivariate case (\cite{einmahl_dehaan_huang_93}, \cite{einmahl_97}, \cite{einmahl_01} and \cite{einmahl_segers_09}) and in moderate dimensions (\cite{coles_tawn_91} and \cite{sabourin_naveau_fougeres_2013}). Inference on the spectral measure has also been studied in a Bayesian framework, for instance by \cite{guillotte_perron_segers_11}. In higher dimensions, mixtures of Dirichlet distributions are often used to model the spectral densities (\cite{boldi_davison_2007}, \cite{sabourin_naveau_14}). Some alternative approaches based on grid estimators (\cite{lehtomaa_resnick}) and Principal Component Analysis (\cite{cooley_thibaud_2019}, \cite{sabourin_drees}) have also been recently proposed.

More recently, the study of the spectral measure's support has become an active topic of research. In high dimensions, it is likely that this measure only places mass on low-dimensional subspaces. The measure is then said to be \emph{sparse}. Sparsity arises all the more for standard regular variation. There, it is possible that the marginals of $\boldX$ are not tail equivalent and therefore that the support of the spectral measure is included in $\mathbb S^{r-1}_+$ for $r \ll d$. This is the approach we use in this article. Then, identifying low-dimensional subspaces on which the spectral measure puts mass allows one to capture clusters of directions which are likely to be extreme together (\cite{chautru_2015}, \cite{janssen_wan}). However, for such subspaces the convergence in \eqref{eq:reg_var_intro} often fails for topological reasons which makes the identification of these subsets challenging. 
 
 Several algorithms have been recently proposed to identify the extremal directions of $\boldX$. \cite{goix_sabourin_clemencon_17} consider $\epsilon$-thickened rectangles to estimate the directions on which the spectral measure concentrates. This estimation is based on a tolerance parameter $\epsilon > 0$ and brings out a sparse representation of the dependence structure. The authors provide an algorithm called DAMEX (for Detecting Anomalies among Multivariate EXtremes) of complexity $O(d n \log n)$, where $n$ corresponds to the number of data points. Subsequently, \cite{chiapino_sabourin_2016} propose an incremental-type algorithm (CLustering Extreme Features, CLEF) to group components which may be large together. This algorithm is based on the DAMEX algorithm and also requires a hyperparameter $\kappa_{\min}$. Several variants of the CLEF algorithm have then been proposed by \cite{chiapino_sabourin_segers_2019}. These approaches differ in the stopping criteria which are based on asymptotic results of the coefficient of tail dependence. A $O(d n \log n)$ complexity has also been reached by \cite{simpson_et_al} who base their method on hidden regular variation.

Since the self-normalized vector $\boldX / |\boldX|$ fails to capture the support of a sparse spectral measure $S$, we replace it by another angular component based on the Euclidean projection onto the simplex $\mathbb S^{d-1}_+ = \{\boldx \in \R^d_+ : x_1 + \ldots + x_d =1 \}$. This projection has been widely studied in learning theory (see e.g. \cite{duchi_et_al_2008}, \cite{liu_ye_09}, or \cite{kyrillidis_et_al_2013}). Many different efficient algorithms have been proposed, for instance by \cite{duchi_et_al_2008} and \cite{condat_2016}. Based on this projection, we define the concept of sparse regular variation. With this approach we obtain a new angular limit vector $\boldZ$ whose distribution slightly differs from the spectral measure and which is more likely to be sparse. We prove that under mild conditions both concepts of regular variation are equivalent and we give the relation between both limit vectors. Besides, we study this new angular limit and show that it allows one to capture the extremal directions of $\boldX$. The numerical results we provide emphasize the efficiency of our method to detect directions which may be large together. These results also highlight how the new vector $\boldZ$ provides an interpretation of the relative importance of a coordinate $j$ in a cluster of extremal directions.


\paragraph{Outline} The structure of this paper is as follows. We introduce in Section \ref{sec:sparse_reg_var} the concept of sparse regular variation based on the Euclidean projection onto the simplex. This notion gives rise to a new angular vector, for which we study its distribution and compare it with the spectral measure. In the main theorem of this section we establish the equivalence under mild conditions between sparse regular variation and standard regular variation. In Section \ref{sec:detection_extremal_directions_sparse_reg_var} we discuss to what extent the Euclidean projection allows us to better capture the extremal directions of $\boldX$. We use a natural partition of the simplex to address this issue and we prove that $\boldTheta$ and $\boldZ$ similarly behave on so-called maximal directions. Finally we illustrate in Section \ref{sec:numerical_results} the performance of our method on simulated data and briefly discuss it with the approach of \cite{goix_sabourin_clemencon_17}.

\paragraph{Notation}
We introduce some standard notation that is used throughout the paper. Symbols in bold such as $\boldx \in \R^d$ are vectors with components denoted by $x_j$, $j \in \{1, \ldots, d\}$. Operations and relationships involving such vectors are meant componentwise. We define $\boldo = (0,\ldots,0) \in \R^d$, $\R^d_+ = \{\boldx \in \R^d : \boldx \geq \boldo\}$, and $B^d_+(0,1) = \{ \boldx \in \R^d_+ : x_1 + \ldots + x_d \leq 1\}$. For $j = 1,\ldots, d$, $\bolde_j$ denotes the $j$-th vector of the canonical basis of $\R^d$. For $a \in \R$, $a_+$ denotes the positive part of $a$, that is $a_+ = a$ if $a \geq 0$ and $a_+ = 0$ otherwise. If $\boldx \in \R^d$ and $\beta =\{i_1, \ldots, i_r\} \subset \{1, \ldots, d\}$, then $\boldx_\beta$ denotes the vector $(x_{i_1}, ..., x_{i_r})$ of $\R^r$. For $p \in [1, \infty]$, we denote by $|\cdot|_p$ the $\ell^p$-norm in $\R^d$, except for $p=1$ for which we just write $|\cdot|$. We write $\overset{w}{\rightarrow}$ for the weak convergence. For a set $E$, we denote by ${\cal P}(E)$ its power set: ${\cal P}(E) = \{A, \, A \subset E\}$. We also use the notation ${\cal P}^*(E) = {\cal P}(E) \setminus \{\emptyset \}$. If $E = \{1, \ldots, r\}$, we simply write ${\cal P}_r = {\cal P} (\{1, \ldots, r\} )$ and ${\cal P}_r^* = {\cal P} (\{1, \ldots, r\} ) \setminus \{\emptyset\}$. For a finite set $E$, we denote by $|E|$ its cardinality. Finally, if $F$ is a subset of a set $E$, we denote by $F^c$ the complementary set of $F$ in $E$.

\section{Sparse regular variation} \label{sec:sparse_reg_var}

\subsection{From standard to sparse regular variation} \label{subsec:projection_onto_simplex}

We start from Equation \eqref{eq:reg_var_intro} which we rephrase as follows:
\begin{equation}\label{eq:reg_var_spectral_vector}
\P \big( (|\boldX|/t, \boldX / |\boldX|) \in \cdot \mid |\boldX| > t \big) \stackrel{w}{\to} \P( (Y, \boldTheta) \in \cdot )\, , \quad t \to \infty\, ,
\end{equation}
where $Y$ is a Pareto($\alpha$)-distributed random variable independent of $\boldTheta$.
We call the random vector $\boldTheta$ the \emph{spectral vector}, its distribution being the spectral measure. If the convergence \eqref{eq:reg_var_spectral_vector} holds we write $\boldX \in \RV(\alpha, \boldTheta)$. In many cases the spectral measure is sparse, that is, it places mass on some lower-dimensional subspaces. The self-normalized vector $\boldX/|\boldX|$ then fails to estimate the spectral vector $\boldTheta$ on such subsets.

\begin{rem} \label{rem:notion_of_sparsity}
	The notion of sparsity in EVT can be defined in two different ways. The first one concerns the number of subsets $\{\boldx \in \R^d_+ : \boldx_\beta > \boldo \text{ and } \boldx_{\beta^c} = \boldo\}$, $\beta \in {\cal P}_d^*$, which gather the mass of the spectral measure (see Section \ref{sec:detection_extremal_directions_sparse_reg_var} for some insights on these subsets). "Sparse" means then that this number is much smaller than $2^d-1$. This is for instance the device of \cite{goix_sabourin_clemencon_17}. It corresponds to the assumption (S2.a) in \cite{engelke_ivanovs_2020}. The second notion deals with the number of null coordinates in the spectral vector $\boldTheta$. In this case, "sparse" means that with high probability $|\boldTheta|_0 \ll d$, where $|\cdot|_0$ denotes the $\ell^0$-norm of $\boldTheta$, that is, $|\boldTheta|_0 = |\{i = 1, \ldots, d : \theta_i \neq 0\}|$. This is denoted by (S2.b) in \cite{engelke_ivanovs_2020}. In all this article we refer to this second notion.
\end{rem}

In order to better capture this sparsity, we replace the quantity $\boldX/|\boldX|$ by the vector $\pi_1(\boldX/t)$, where $\pi_z(\boldv)$ denotes the Euclidean projection of the vector $\boldv$ in $\R^d_+$ onto the positive sphere $\mathbb S^{d-1}_+(z) := \{\boldx \in \R^d_+ : x_1 + \ldots + x_d = z\}$, $z> 0$. The projected vector $\pi_z(\boldv)$ is defined as the unique solution of the following optimization problem (see \cite{duchi_et_al_2008} and the references therein):
\begin{equation} \label{minimization_problem_projection}
\underset{\boldw}{\mbox{minimize }} \frac{1}{2} |\boldw - \boldv|_2^2 \quad \mbox{s.t.} \quad |\boldw|_1 = z\, .
\end{equation}
The \emph{Euclidean projection onto the positive sphere $\mathbb S^{d-1}_+(z)$} is then defined as the application
\[
\begin{array}{ccccl}
\pi_z & : & \R^d_+ & \to & \mathbb S^{d-1}_+(z) \\
& & \boldv & \mapsto & \boldw = (\boldv-\lambda_{\boldv,z})_+\, ,
\end{array}
\]
where $\lambda_{\boldv, z}$ is the unique constant satisfying the relation $\sum_{i=1}^d (v_i - \lambda_{\boldv,z})_+ =z$. Several algorithms which compute $\pi_z(\boldv)$ have been introduced (\cite{duchi_et_al_2008}, \cite{condat_2016}). We present two of them in Appendix \ref{sec:appendix_algo}, a first one which gives an intuitive way to compute $\pi_z(\boldv)$ and a second one based on a median-search procedure whose expected time complexity is $O(d)$.


The projection satisfies the relation $ \pi_z(\boldv) = z\pi_1(\boldv/z) $ for all $\boldv \in \R^d_+$ and $z > 0$. This is why we mainly focus on the projection $\pi_1$ onto the simplex $\mathbb S^{d-1}_+$. In this case we shortly write $\pi$ for $\pi_1$. An illustration of $\pi$ for $d=2$ is given in Figure \ref{fig_projection}. We list below some straightforward results satisfied by the projection.
\begin{enumerate}
	\item[P1.] \label{item:properties_projection_order_coordinates} The projection preserves the order of the coordinates: If $v_{\sigma(1)} \geq \ldots \geq v_{\sigma(d)}$ for a permutation $\sigma$, then $\pi(\boldv)_{\sigma(1)} \geq \ldots \geq \pi(\boldv)_{\sigma(d)}$ for the same permutation.
	\item[P2.] \label{item:properties_projection_positif} If $\pi(\boldv)_j > 0$, then $v_j > 0$. Equivalently, $v_j = 0$ implies $\pi(\boldv)_j = 0$.
	\item[P3.] The projection $\pi$ is continuous, as every projection on a convex, closed set in a Hilbert space.
\end{enumerate}

\begin{figure}[!th]
	\begin{center}
		\begin{tikzpicture}[scale=1.8]
		
		\tikzset{cross/.style={cross out, draw=black, minimum size=2*(#1-\pgflinewidth), inner sep=0pt, outer sep=0pt},
			cross/.default={2pt}}
		
		\draw (0,0) node[below] {$O$};
		\draw (2,0) node[below] {$\bolde_1$};
		\draw (0,2) node[left] {$\bolde_2$};
		\draw (0,1) node[left] {$1$};
		\draw (1,0) node[below] {$1$};
		
		\draw[->] (0,0) -- (2,0) ;
		\draw[->] (0,0) -- (0,2);
		
		\draw[densely dotted] (1,0) -- (2,1) ;
		\draw[densely dotted] (0,1) -- (1,2);
		\draw (1,0) -- (0,1);
		
		\draw[->][black] (0,0) -- (0.7,2);
		\draw[black, dashed] (0.7,2) -- (0,1);
		\draw (0.7,2) node[right] {$\boldu$};
		\draw (0,1) node[right] {$\pi(\boldu)$};
		\draw[->][black] (0,0) -- (0,1);

		\draw[->][black] (0,0) -- (1.5, 1);
		\draw[black, dashed] (1.5, 1) -- (0.75,0.25);
		\draw (1.5,1) node[right] {$\boldv$};
		\draw (0.75,0.25) node[right] {$\pi(\boldv)$};
		\draw[->][black] (0,0) -- (0.75, 0.25);
		
		\end{tikzpicture}
		
	\end{center}
	\caption{The Euclidean projection onto the simplex $\mathbb S^{1}_+$.}
	\label{fig_projection}
\end{figure}
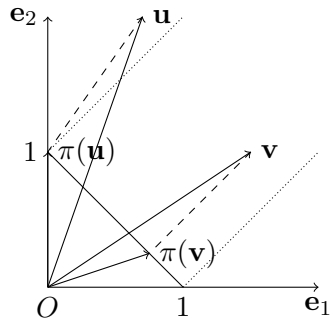

The substitution of the self-normalized vector $\boldX/|\boldX|$ by the projected vector $\pi(\boldX/t)$ motivates the following definition.

\begin{defin}[Sparse regular variation]\label{def:sparse_reg_var}
	A random vector $\boldX \in \R^d_+$ is sparsely regularly varying if there exist a random vector $\boldZ$ defined on the simplex $\mathbb{S}^{d-1}_+$ and a nondegenerate random variable $Y$ such that
	\begin{equation} \label{eq:sparse_reg_var}
	\P \big( ( |\boldX| / t , \pi (\boldX / t ) ) \in \cdot \mid |\boldX| > t \big) \overset{w}{\to} \P((Y, \boldZ) \in \cdot)\, , \quad t \rightarrow \infty\,.
	\end{equation}
	In this case we write $\boldX \in \SRV(\alpha, \boldZ)$.
\end{defin}
In this case, standard results on regularly varying random variables state that there exists $\alpha > 0$ such that $Y$ follows a Pareto distribution with parameter $\alpha$. The continuity of $\pi$ ensures that regular variation with limit $(Y, \boldTheta)$ implies sparse regular variation with limit $(Y, \pi(Y \boldTheta))$. Contrary to the limit in \eqref{eq:reg_var_spectral_vector}, we lose independence between the radial component $Y$ and the angular component $\boldZ$ of the limit. The dependence relation between both components is given by the following proposition.

\begin{prop} \label{prop:dependence_Z_Y}
	If $\boldX \in \SRV(\alpha, \boldZ)$, then for all $r \geq 1$ we have
	\[
	\P(\boldZ \in \cdot \mid Y > r) \overset{d}{=} \P( \pi(r \boldZ) \in \cdot) \,.
	\]
\end{prop}

For $\beta \in {\cal P}_d^*$ we denote by $\bolde(\beta)$ the vector with $1$ in position $j$ if $j \in \beta$ and $0$ otherwise. Then the vector $\bolde(\beta)/|\beta|$ belongs to the simplex. We consider the following class of discrete distributions on the simplex:
\begin{equation} \label{eq:class_discrete_distributions}
\sum_{ \beta \in {\cal P}_d^* } p(\beta) \, \delta_{\bolde(\beta) / |\beta|}\, ,
\end{equation}
where $(p(\beta))_\beta$ is a $2^d-1$ vector with nonnegative components summing to 1. This is the device developed by \cite{segers_12}.  The family of distributions \eqref{eq:class_discrete_distributions} is stable under multiplication by a positive random variable and Euclidean projection onto the simplex. Hence, if $\boldTheta$ has a distribution of type \eqref{eq:class_discrete_distributions}, then $\boldZ = \boldTheta$ a.s. The following corollary states that these kind of distributions are the only possible discrete distributions for $\boldZ$.

\begin{corr}\label{corr:Z_points_e_beta}
	If the distribution of $\boldZ$ is discrete, then it is of the form \eqref{eq:class_discrete_distributions}.
\end{corr}

The family of distributions given in \eqref{eq:class_discrete_distributions} forms an accurate model for the angular vector $\boldZ$. Indeed, the distributions of this class place mass on some particular points of the simplex on which extremes values often concentrate in practice. It includes the case of complete dependence which corresponds to the case $p( \{1, \ldots, d \} ) = 1$ and the case of asymptotic independence which corresponds to the case $p(\{j\}) = 1/d$ for all $j = 1, \ldots, d$ (see also Example \ref{ex:asympt_indep}).


\subsection{Main result}

We establish in this section some relations between both vectors $\boldTheta$ and $\boldZ$. We start this section by introducing some notation which will be useful throughout the article. We consider the sets
\[
A_\boldx = \{\boldu \in \mathbb S^{d-1}_+ : \boldu \geq \boldx\}\, , \quad \boldx \in B_+^d(0,1)\, ,
\]
\[
{\cal X}_\beta = \{\boldx \in B_+^d(0,1) : \boldx_\beta > \boldo_\beta, \, \boldx_{\beta^c} = \boldo_{\beta^c} \}\, , \quad \beta \in {\cal P}_d^*\, ,
\]
\[
{\cal X}_\beta^0 = \{\boldx \in B_+^d(0,1) : \boldx_{\beta^c} = \boldo_{\beta^c} \}\, , \quad \beta \in {\cal P}_d^*\, ,
\]
and we define $\lambda_\beta$ as the Lebesgue measure on the set ${\cal X}_\beta$. For $\beta, \gamma \in {\cal P}_d^*$ such that $\gamma \supset \beta$ we also consider the functions
\[G_\beta(\boldx) = \P( \boldZ_\beta > \boldx_\beta, \, \boldZ_{\beta^c} \leq \boldx_{\beta^c})\, , \quad \boldx_\beta \in B^{|\beta|}_+(0,1)\, , \quad \boldx_{\beta^c} \in B^{|\beta^c|}_+(0,1) \, ,
\]
and
\[
H_{\beta, \gamma}(\boldu, v, w)
= \P \big( \phi_\gamma(\boldZ)_\beta \geq \boldu, \,
\min_{j \in \gamma \setminus \beta} \phi_\gamma(\boldZ)_j > v, \,
\max_{j \in \gamma^c} \phi_\gamma(\boldZ)_j \leq w \big)\, ,
\quad \boldu \in B_+^{|\beta|}(0,1)\, , \quad v, w \in [0,1]\, ,
\]
where $\phi_\gamma : \mathbb S^{d-1}_+ \to \R^d_+, \, \boldu \mapsto \phi_\gamma(\boldu)$ is defined as
\[
\phi_\gamma(\boldu)_j
=
\left\{
\begin{array}{ll}
u_j + \frac{ |\boldu_{\gamma^c}| }{ |\gamma| }\, , \quad \text{for }j \in \gamma\, ,\\
u_j + \frac{ |\boldu_{\gamma^c\setminus\{j\}}| }{ |\gamma|+1 }\, , \quad \text{for }j \in \gamma^c\, ,
\end{array}
\right.
\]
see Lemma \ref{lem:proj_beta_betac} for more insights on this function.

We consider the following assumption:
\begin{enumerate}
	\item[(A)] For all $\beta, \gamma \in {\cal P}_d^*$ such that $\gamma \supset \beta$ and for $\lambda_\beta$-almost every $\boldx \in {\cal X}_\beta$ the function $H_{\beta, \gamma}$ is continuously differentiable in $(\boldx_\beta,0,0)$.
\end{enumerate}

\begin{rem}[On assumption (A)]\label{rem:assumption_A}
	Suppose that the distribution of $\boldZ$ is a mixture of a discrete part $\sum_k a_k \delta_{\boldu_k}$ and a continuous part $\sum_{\beta \in P} a_\beta f_\beta$ with $P \subset \{ \beta \in {\cal P}_d^* : |\beta| \geq 2\}$ and with continuous densities $f_\beta$. Then, so is the distribution of $\phi_\gamma(\boldZ)$, since $\phi_\gamma(\boldZ)$ is a linear transformation of $\boldZ$. In this case, assumption (A) is satisfied.
\end{rem}

We are now able to state the main result of this paper.

\begin{theo}[Equivalence of regular variation and sparse regular variation] \label{theo:relation_G_Theta} Under assumption (A) we have the equivalence between regular variation and sparse regular variation. More precisely:
	\begin{enumerate}
		\item If $\boldX\in \RV(\alpha, \boldTheta)$, then $\boldX \in \SRV(\alpha, \boldZ)$, with $\boldZ = \pi(Y\boldTheta)$ satisfying
		\begin{equation}\label{eq:Z_terms_of_Theta}
		G_\beta(\boldx)
		= \E \bigg[ \bigg(
		1 \wedge
		\min _{j \in \beta_+} \Big( \frac{|\beta| \Theta_j - |\boldTheta_\beta|}{|\beta| x_j - 1}\Big)_+^\alpha
		\wedge
		\min_{j \in \beta^c}(|\boldTheta_\beta| - |\beta| \Theta_j)_+^\alpha
		- \max_{j \in \beta_-}\Big( \frac{|\beta| \Theta_j - |\boldTheta_\beta|}{|\beta| x_j - 1}\Big)_+^\alpha
		\bigg)_+ \bigg]\, ,
		\end{equation}
		for all $\boldx \in {\cal X}_\beta^0$ such that for all $j \in \beta$, $x_j \neq 1/|\beta|$, and where $\beta_+ = \{j \in \beta, \, x_j > 1/|\beta|\}$ and $\beta_- =  \{j \in \beta, \, x_j < 1/|\beta|\}$.
	\item If $\boldX \in \SRV(\alpha, \boldZ)$ with $\boldZ$ satisfying (A), then $\boldX \in \RV(\alpha, \boldTheta)$ with $\boldTheta$ satisfying
	\begin{equation} \label{eq:relation_Theta_Z}
	\P(\boldTheta \in A_\boldx)
	= \P(\boldZ \in A_\boldx)
	+ \alpha^{-1} \sum_{\gamma \supset \beta} \mathrm d H_{\beta, \gamma}(\boldx_\beta, 0, 0) \cdot \Big(\boldx_\beta - \frac{1}{|\gamma|}, -\frac{1}{|\gamma|}, -\frac{1}{|\gamma|+1} \Big)\, ,
	\end{equation}
	for $\beta \supset \gamma \in {\cal P}_d^*$ and $\lambda_\beta$-almost every $\boldx \in {\cal X}_\beta$.
\end{enumerate}
\end{theo}

\begin{rem}[Discrete distributions]
Note that if the angular vector $\boldZ$ has a discrete distribution, then the assumption (A) is satisfied and we have $\mathrm d H(\boldx_\beta, 0,0) = 0$ for all $\beta \subset \gamma$ and $\lambda_\beta$-almost every $\boldx \in {\cal X}_\beta$. Hence Theorem \ref{theo:relation_G_Theta} ensures that there exists a spectral vector $\boldTheta$ such that $\P( \boldTheta \in A_\boldx) = \P(\boldZ \in A_\boldx)$, i.e. $\boldTheta \stackrel{d}{=}\boldZ$. Actually, using Corollary \ref{corr:Z_points_e_beta} and similar arguments than above we obtain $\boldTheta = \boldZ$ a.s.
\end{rem}

Equation \eqref{eq:Z_terms_of_Theta} gives a relation between the distribution of $\boldZ$ and the one of $\boldTheta$. While its complexity makes it difficult to use in all generality, specific choices for $\boldx$ lead to useful results. A convenient particular case is the one with $\beta = \{1, \ldots, d\}$ and $\boldx < 1/d$ which provides the relation
\[
G_\boldZ(\boldx)
= \E \bigg[ \bigg( 1 - \max_{1 \leq j \leq d} \bigg(\frac{1/d - \Theta_j}{1/d-x_j} \bigg)^\alpha\bigg)_+ \bigg] \, .
\]
In particular for $\boldx = \boldo$ we obtain
\begin{equation} \label{eq:Z_all_components_positive}
G_\boldZ(\boldo) = 1 - \E \Big[ \max_{1\leq j \leq d } (1-d\Theta_j)^\alpha \Big] \, .
\end{equation}
Thus, the probability for $\boldZ$ to have a null component is 
\begin{equation} \label{eq:Z_one_null_component}
\P(Z_j=0 \text{ for some } j \leq d) = \E \Big[ \max_{1 \leq j \leq d } (1-d \Theta_j)^\alpha \Big]\,.
\end{equation}
This quantity is null if and only if for all $j = 1, \ldots, d$, $\Theta_j=1/d$ a.s. and is equal to $1$ if and only if $\min_{1 \leq j \leq d}\Theta_j=0$ a.s. This implies that the new angular vector $\boldZ$ is more likely to be sparse in the sense of Remark \ref{rem:notion_of_sparsity}. In particular, all usual spectral models on $\boldTheta$ that are not supported on the axis are not suitable for $\boldZ$. More insights into the sparsity of the vector $\boldZ$ is given in Section \ref{sec:detection_extremal_directions_sparse_reg_var}.

\begin{ex}\label{ex:uniform}
	We consider a spectral vector $\boldTheta$ in $\mathbb S^1_+$ with a first component $\Theta_1$ uniformly distributed (and then $\Theta_2 = 1 - \Theta_1$ is also uniformly distributed). This fits into the framework of Remark \ref{rem:assumption_A}. We also assume $\alpha = 1$ for simplicity. The probability that $\boldZ$ belongs to the first axis is equal to $\P(Z_1 = 1) = \P(Y \Theta_1 - Y \Theta_2 \geq 1) = \P(2 \Theta_1 - 1 \geq 1/Y)$ (see Lemma \ref{lem:proj_beta_betac}). Since the random variable $1/Y$ follows a uniform distribution on $(0,1)$ and is independent of $\boldTheta$ we obtain
	\begin{equation}\label{eq:Z_faces_uniform}
	\P(Z_1 = 1) = \int_0^1 \P(2 \Theta_1 - 1 \geq u) \, \mathrm d u = \int_0^1 \P \Big(\Theta_1 \geq \frac{u+1}{2} \Big) \, \mathrm d u = \int_0^1 \frac{1-u}{2} \, \mathrm d u = 1/4 \, .
	\end{equation}
	Besides, if $x \in (0,1)$, Lemma \ref{lem:proj_beta_betac} entails that
	\begin{align*}
	\P(0 < Z_1 \leq x) = \P(0 < Y \Theta_1 -(Y-1)/2 \leq x) = \int_0^1 \P\big[ (1-u)/2 < \Theta_1 \leq (1 +(2x-1)u) / 2 \big] \, \mathrm d u = x/2\, .
	\end{align*}
	The distribution of $Z_1$ is thus given by $\delta_0/4 + \delta_1/4 + U(0,1)/2$, where $U$ denotes a uniform distribution on $(0,1)$ and $\delta_a$ a Dirac measure at point $a$. We check in Appendix \ref{app:calculations} that this vector $\boldZ$ satisfies Equations \eqref{eq:Z_terms_of_Theta} and \eqref{eq:relation_Theta_Z}.
\end{ex}

\section{Detection of extremal directions with sparse regular variation} \label{sec:detection_extremal_directions_sparse_reg_var}

\subsection{Sparsity in extremes}

This section tackles the issue of detecting extremal directions for a regularly varying random vector $\boldX$. In such a context it is helpful to partition the underlying space, in our case $\mathbb S ^{d-1}_+$, with understandable subsets (\cite{chautru_2015}, \cite{goix_sabourin_clemencon_17}, \cite{simpson_et_al}). In this article we consider the subsets
\begin{equation} \label{def:cones_C}
C_\beta = \big \{ \boldx \in \mathbb S^d_+ : x_i > 0 \text{ for } i \in \beta, \, x_i = 0 \text{ for } i \notin \beta \big\} \, , \quad \beta \in {\cal P}_d^*\, ,
\end{equation}
which form a partition of the simplex. An illustration of these subsets in dimension $3$ is given in Figure \ref{fig:subcones}. This partition is helpful to study the tail structure of $\boldX$. Indeed, for $\beta \in {\cal P}_d^*$ the inequality $\P(\boldTheta \in C_\beta) > 0$ means that it is likely to observe simultaneously large values in the directions $i \in \beta$ and small values in the directions $i \in \beta^c$. Then, identifying the subsets $C_\beta$ which concentrate the mass of the spectral measure allows us to bring out clusters of coordinates which can be simultaneously large.

\begin{figure}[!th]
	\begin{center}
		\begin{tikzpicture}[scale=2]
		
		\tikzset{cross/.style={cross out, draw=black, minimum size=2*(#1-\pgflinewidth), inner sep=0pt, outer sep=0pt},
			cross/.default={2pt}}
		
		\draw[dashed] (1,0,0) -- (0,1,0);
		\draw[dashed] (0,1,0) -- (0,0,1);
		\draw[dashed] (0,0,1) -- (1,0,0);
		
		\draw[->] (0,0,0) -- (1.5,0,0) ;
		\draw[->] (0,0,0) -- (0,1.5,0);
		\draw[->] (0,0,0) -- (0,0,1.5);
		
		\draw (0,0,0) node[left] {$O$};
		\draw (0,0,1.5) node[left] {$\bolde_1$};
		\draw (1.5,0,0) node[below] {$\bolde_2$};
		\draw (0,1.5,0) node[left] {$\bolde_3$};
		
		\draw (0,0,1) node[left] {$1$};
		\draw[black] (0,0,1) node {$\bullet$};
		\draw (1,0,0) node[below] {$1$};
		\draw[black] (1,0,0) node {$\bullet$};
		\draw (0,1,0) node[left] {$1$};
		\draw[black] (0,1,0) node {$\bullet$};
	
		\end{tikzpicture}
	\end{center}
	\caption{The subsets $C_\beta$ in dimension $3$ for the $\ell^1$-norm. The subsets $C_{ \{1\} }$, $C_{ \{2\} }$, and $C_{ \{3\} }$ correspond to the unit vectors $\bolde_1$, $\bolde_2$, and $\bolde_3$, respectively. The dashed lines indicate the subsets $C_{ \{1, 2\} }$, $C_{ \{1, 3\} }$, and $C_{ \{2, 3\} }$. Finally, the subset $C_{\{1, 2, 3 \}}$ correspond to the interior of the simplex.} \label{fig:subcones}
\end{figure}
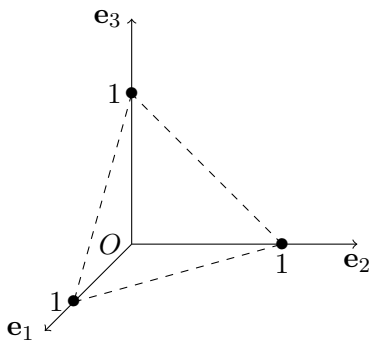

\begin{rem} \label{rem:coord_system}
	Our approach aims to detect sparse directions that are aligned with the standard coordinate system. For $\boldX \in \R^d_+$ it allows us to understand how a marginal affects the extremal behavior of $\boldX$. In applications, if $\boldX$ represents a phenomenon, the goal is to understand which groups of marginals are the main causes of the extremal behavior of this phenomenon. This is why we do not focus on what happens if the directions do not align with the standard coordinate system.
\end{rem}

\begin{ex}\label{ex:asympt_indep}
	A standard example of sparsity is asymptotic independence for which the spectral measure only puts mass on the axis: $\P(\boldTheta \in \sqcup_{1 \leq j \leq d}\, \{\bolde_j \} ) = \P(\boldTheta \in \sqcup_{1 \leq j \leq d} \, C_{ \{j\} } ) =  1$. This means that there is never more than one direction which contributes to the extremal behavior of the vector. This concept has been studied by many authors, see for instance \cite{ledford_tawn_96} or \cite{ramos_ledford_09}.
\end{ex}

As for any low-dimensional subspaces, topological issues may arise for the $C_\beta$'s in the convergence \eqref{eq:reg_var_spectral_vector}. Indeed, for $\beta \neq \{1, \ldots, d\}$ the subset $C_\beta$ is included in its boundary (with respect to the topology of the simplex) and the convergence \eqref{eq:reg_var_spectral_vector} fails for such a set. This kind of problems appears since the spectral measure may put mass on low-dimensional subspaces while the data generally do not concentrate on such subspaces. This issue can be circumvent with sparse regular variation. We replace the study of the sets $\{ \boldx \in \R^d_+ : |\boldx| > 1, \, \boldx / |\boldx| \in C_\beta\}$ by the one of the sets $\{ \boldx \in \R^d_+ : |\boldx| > 1, \, \pi(\boldx / t) \in C_\beta\}$ which enjoy better topological properties.

In this context, many other sets have been proposed recently. \cite{goix_sabourin_clemencon_2016} define the truncated $\epsilon$-cones as
\[
\big \{ \boldx \in \R^d_+ : |\boldx|_\infty > 1, \, x_i > \epsilon |\boldx|_\infty \text{ for } i \in \beta, \, x_i \leq \epsilon  |\boldx|_\infty \text{ for } i \notin \beta \big\} \, .
\]
Subsequently, \cite{goix_sabourin_clemencon_17} introduce the notion of $\epsilon$-thickened rectangles:
\[
\big \{ \boldx \in \R^d_+ : |\boldx|_\infty > 1, \, x_i > \epsilon \text{ for } i \in \beta, \, x_i \leq \epsilon \text{ for } i \notin \beta \big\} \, .
\]
\cite{chiapino_sabourin_2016} relax the condition on $\beta^c$ and define the rectangles
\[
\big \{ \boldx \in \R^d_+ : x_i > 1 \text{ for } i \in \beta \big\}\, ,
\]
to focus on groups of variables that may be large together. Similarly to the ideas of \cite{chiapino_sabourin_2016} we relax the condition on $\beta$ in the definition of the $C_\beta$'s which leads to the study of the subsets $\{\boldx \in \mathbb S^{d-1}_+ : x_i = 0 \text{ for } i \notin \beta\} = \mathbb S ^{d-1}_+ \cap \Vect(\bolde_j, j \in \beta)$.

We gather in the following proposition some results regarding the behavior of $\pi(\boldX/t)$ and $\boldZ$ on the subsets $C_\beta$ and $\mathbb S^{d-1}_+ \cap \Vect(\bolde_j, j \in \beta)$. In all what follows we write $\Theta_{\beta, \, j}$ (resp. $\Theta_{\beta, \, j, \, +}$) for $\sum_{k \in \beta} (\Theta_k - \Theta_j)$ (resp. $\sum_{k \in \beta} (\Theta_k - \Theta_j)_+$).

\begin{prop}\label{prop:Z_on_specific_subsets}
Let $\boldX \in \RV(\alpha, \boldTheta)$ and set $\boldZ = \pi(Y \boldTheta)$, where $Y$ is a Pareto($\alpha$)-distributed random variable independent of $\boldTheta$.
\begin{enumerate}
	\item For any $\beta \in {\cal P}_d^*$ we have
	\begin{equation} \label{eq:cv_projection_Z_on_C_beta}
		\P(\pi(\boldX/t) \in C_\beta \mid |\boldX| > t ) \to \P(\boldZ \in C_\beta)\, ,
		\quad t \to \infty\, ,
	\end{equation}
	and
	\begin{equation} \label{eq:cv_projection_Z_on_beta}
		\P(\pi(\boldX/t)_{\beta^c} = 0 \mid |\boldX| > t ) \to \P(\boldZ _{\beta^c} = 0)\, ,
		\quad t \to \infty\, .
	\end{equation}
	\item For any $\beta \in {\cal P}_d^*$ we have
	\begin{equation} \label{eq:Z_C_beta}
	\P(\boldZ \in C_\beta)
	= \E\Big[ \Big(
	\min_{j \in \beta^c} \Theta_{\beta, \, j, \, +}^\alpha
	- \max_{j \in \beta} \Theta_{\beta, \, j, \, +}^\alpha
	\Big)_+ \Big] \, ,
	\end{equation}
	and
	\begin{equation} \label{eq:Z_beta_null_components}
	\P(\boldZ_{\beta^c} = 0)
	= \E\Big[ \min_{j \in \beta^c}
	\Theta_{\beta, \, j, \, +}^\alpha
	\Big]\, .
	\end{equation}
\end{enumerate}
\end{prop}

Regarding the behavior of $\pi(\boldX/t)$, the convergence \eqref{eq:sparse_reg_var} holds for any pair of Borel sets $A \times B \in (1, \infty) \times \mathbb S^{d-1}_+$ such that $\P(Y \boldTheta \in \partial \pi^{-1}(B)) = 0$, where $\partial \pi^{-1}(B)$ denotes the boundary of the set $\pi^{-1}(B)$. The first point of Proposition \ref{prop:Z_on_specific_subsets} states that the subsets $C_\beta$ and $\mathbb S^{d-1}_+ \cap \Vect(\bolde_j, j \in \beta)$ satisfy this condition. This implies in particular that the sparsity of $\boldZ$ can be studied through the projected vector $\pi(\boldX/t)$. This will be illustrated in Section \ref{sec:numerical_results}.

The second point of Proposition \ref{prop:Z_on_specific_subsets} provides some interesting relations between the sparsity of $\boldZ$ and $\boldTheta$. If we consider $\beta = \{1,\ldots,d\}$, then we obtain the probability that all coordinates are positive which has already been computed in \eqref{eq:Z_all_components_positive}. It is equal to $G_\boldZ(\boldo) = 1 - \E \left[ \max_{1\leq j \leq d } (1-d\Theta_j)^\alpha \right]$. Another particular case of Equations \eqref{eq:Z_C_beta} and \eqref{eq:Z_beta_null_components} is the one where $\beta$ corresponds to a single coordinate $\{j_0\}$. In this case, since $\boldZ$ belongs to the simplex, both probabilities $\P(\boldZ_{\beta^c} = 0)$ and $\P(\boldZ \in C_\beta)$ are equal. Their common value corresponds to the probability that $\boldZ$ concentrates on the $j_0$-th axis, which is equal to
\begin{equation} \label{eq:Z_one_positive_component}
\P( Z_{j_0} = 1) = \E \Big[ \min_{j \neq j_0} (\Theta_{j_0} - \Theta_j)_+^\alpha \Big] \, .
\end{equation}
Then, Equation \eqref{eq:Z_one_positive_component} can be developed in the following way:
\begin{align*}
\P( Z_{j_0} = 1)
& = \E \Big[ \min_{j \neq j_0} (\Theta_{j_0} - \Theta_j)_+^\alpha \indic_{\{\Theta_{j_0} = 1\}} \Big]
+ \E \Big[ \min_{j \neq j_0} (\Theta_{j_0} - \Theta_j)_+^\alpha \indic_{\{\Theta_{j_0} < 1\}} \Big]\\
& = \P(\Theta_{j_0} = 1) + \E \Big[ \min_{j \neq j_0} (\Theta_{j_0} - \Theta_j)_+^\alpha \indic_{\Theta_{j_0} < 1} \Big] \geq \P(\Theta_{j_0} = 1) \, .
\end{align*}
This shows that the vector $\boldZ$ is more likely to be sparse than the spectral vector $\boldTheta$.

\begin{rem} \label{rem:inequality_beta_c}
	Following Equation \eqref{eq:Z_beta_null_components}, we write
	\begin{equation}\label{eq:inequality_beta_c}
	\P(\boldZ_{\beta^c} = 0)
	\geq \E\Big[ \min_{j \in \beta^c} \Theta_{\{1,\ldots,d\}, \, j, \, +}^\alpha \indic_{\{ \boldTheta_{\beta^c} = 0\} } \Big]
	= \E\Big[ \big(\sum_{k \leq d} \Theta_k \big)^\alpha \indic_{\{ \boldTheta_{\beta^c} = 0\} } \Big]
	= \P(\boldTheta_{\beta^c} = 0)\, .
	\end{equation}
	This can also be seen as a direct consequence of Property P2, see Section \ref{subsec:projection_onto_simplex}. This property also gives
	\begin{equation}\label{eq:inequality_beta}
	\P(\boldZ_\beta > 0) \leq \P(\boldTheta_\beta > 0)\, ,
	\end{equation}
	an inequality which will be useful in some proofs.
\end{rem}

\subsection{Maximal directions}\label{subsec:maximal_directions}

We focus in this section on the subsets $C_\beta$. A positive value for $\P( \boldTheta \in C_\beta)$ entails that the marginals $X_j$ for $j \in \beta$ take simultaneously large values while the ones in $\beta^c$ do not (\cite{chautru_2015}, \cite{simpson_et_al}, \cite{goix_sabourin_clemencon_17}). Our aim is to use Proposition \ref{prop:Z_on_specific_subsets} to compare the nullity or not of the probabilities $\P(\boldTheta \in C_\beta)$ and $\P(\boldZ \in C_\beta)$. To this end, it is relevant to focus on the largest group of variables $\beta \in {\cal P}_d^*$ such that $\P(\boldTheta \in C_\beta) > 0$. This motivates the notion of maximal direction.
\begin{defin}[Maximal direction]\label{def:maximal_direction}
	Let $\beta \in {\cal P}_d^*$. We say that a direction $\beta$ is \emph{maximal} for $\boldTheta$ if
	\begin{equation*}
	\P(\boldTheta \in C_\beta) > 0 \quad \text{and} \quad \P(\boldTheta \in C_{\beta'}) =0\, , \text{ for all } \beta' \supsetneq \beta\, .
	\end{equation*}
\end{defin}
We similarly define maximal directions for $\boldZ$.

\begin{rem} \label{rem:maximal_directions_inclusion}
	A straightforward but useful consequence of Definition \ref{def:maximal_direction} is that each direction $\beta$ such that $\P(\boldTheta \in C_\beta) > 0$ is included in a maximal direction of $\boldTheta$. Indeed, if there exists no $\beta' \supsetneq \beta$ such that $\P(\boldTheta \in C_{\beta'}) =0$, then $\beta$ is maximal itself. If not, we consider $\beta' \supsetneq \beta$ such that $\P(\boldTheta \in C_{\beta'})  > 0$. If $\beta'$ is not maximal, then we repeat this procedure with $\beta'$. Since the length of the $\beta$'s is finite, the procedure stops and provides $\gamma \in {\cal P}_d^*$ such that $\beta \subset \gamma$, $\P(\boldTheta \in C_\gamma) >0$ and $\P(\boldTheta \in C_{\gamma'}) = 0$, for all $\gamma' \supsetneq \gamma$.
\end{rem}

The notion of maximal directions is justified by the following theorem.

\begin{theo}\label{theo:comparison_Theta_Z_maximal_subsets} Let $\beta \in {\cal P}_d^*$.
	\begin{enumerate}
		\item If $\P(\boldTheta \in C_\beta) > 0$, then $\P(\boldZ \in C_\beta) > 0$.
		\item The direction $\beta$ is maximal for $\boldTheta$ if and only if it is maximal for $\boldZ$.
	\end{enumerate}
\end{theo}

Theorem \ref{theo:comparison_Theta_Z_maximal_subsets} implies that we do not lose any information on the extremal directions of $\beta$ by studying $\boldZ$ instead of $\boldTheta$. But it is possible that the distribution of $\boldZ$ puts some mass on a subset $C_\beta$ while the one of $\boldTheta$ does not. In such a case, the associated direction $\beta$ is necessarily non-maximal. 

\begin{ex}\label{ex:mass_Z_not_Theta}
	In Example \ref{ex:uniform} we proved that if $\Theta_1$ follows a uniform distribution on $(0,1)$, then $\P(\boldZ \in C_{ \{1\} } ) = 1/4$ while $\P(\Theta_1 = 1) = 0$. This proves that the direction $\beta = \{1\}$ is non-maximal for $\boldZ$.
\end{ex}

Example \ref{ex:mass_Z_not_Theta} shows there may exist $\beta \in {\cal P}_d^*$ such that $\P(\boldZ \in C_\beta) > 0$ and $\P(\boldTheta \in C_\beta) = 0$. In this case, Theorem \ref{theo:comparison_Theta_Z_maximal_subsets} states that the direction $\beta$ is not maximal for $\boldZ$ since it is not maximal for $\boldTheta$. Following Remark \ref{rem:maximal_directions_inclusion}, we consider a maximal direction $\gamma$ of $\boldZ$ such that $\beta \subset \gamma$. Then Theorem \ref{theo:comparison_Theta_Z_maximal_subsets} states that $\P(\boldTheta \in C_\gamma) > 0$. This means that even if the direction $\beta$ does not gather itself coordinates on which extreme values simultaneously occur, there exists a superset of $\beta$ which actually contains extremes. Thus, $\beta$ still gives information on the study of large events.

A natural procedure to capture the extremal directions of $\boldX$ is then the following one. Based on the Euclidean projection $\pi$ we identify the subsets $C_\beta$ on which the distribution of $\boldZ$ places mass. Hopefully, the selected subsets are low-dimensional. Among these subsets we select the maximal ones which also correspond to the maximal direction regarding the spectral vector $\boldTheta$.

\paragraph{What happens on non-maximal directions?}
While the study of maximal directions is the same for $\boldZ$ and $\boldTheta$, we develop here some ideas which highlight the use of $\boldZ$ regarding non-maximal directions. We consider a direction $\beta \in {\cal P}^*_d$ and assume that the associated subset $C_\beta$ satisfies $\P(\boldZ \in C_\beta) > 0$ and $\P(\boldTheta \in C_\beta) = 0$. Then, the direction $\beta$ is necessary non-maximal for $\boldZ$ and satisfies the following inequalities:
\[
0
< \P(\boldZ \in C_\beta)
= \P( \boldZ_\beta > 0, \boldZ_{\beta^c} = 0)
\leq \P( \boldZ_{\beta^c} = 0)
= \P( \pi(Y \boldTheta)_{\beta^c} = 0)\, .
\]
Following Equation \eqref{eq:Z_beta_c_zero} we obtain that
\begin{equation}\label{eq:non_maximal_subsets}
0 < \P( \pi(Y \boldTheta)_{\beta^c} = 0)
= \P \big( |\beta|^{-1} |\boldTheta_\beta| \geq \max_{i \in \beta^c} \Theta_i + Y^{-1} \big)\, .
\end{equation}
This implies that with positive probability $|\beta|^{-1} \sum_{k \in \beta} \Theta_k \geq \Theta_i$ for all $i \in \beta^c$. If we consider for instance $\beta = \{j\}$, then we obtain that with positive probability $\Theta_j \geq \Theta_i$ for all $i \neq j$. More generally, regarding the vector $\boldX$, Equation \eqref{eq:non_maximal_subsets} yields to
\[
0 < \P \big( |\beta|^{-1} |\boldTheta_\beta| \geq \max_{i \in \beta^c} \Theta_i + Y^{-1} \big)
= \lim\limits_{t \to \infty}
\P \big( |\beta|^{-1} |\boldX_\beta| \geq \max_{i \in \beta^c} X_i + t \mid |\boldX| > t \big)\, .
\]
This means that the extreme values of $\boldX_\beta$ are likely to be larger than the extreme ones of $\boldX_{\beta^c}$. This does not contradict the fact that $\P(\boldTheta \in C_\beta) = 0$ which only implies that it is unlikely to observe simultaneously large values in the direction $\beta$ and small values in the direction $\beta^c$.

Hence, if we detect a maximal direction $\gamma$, we first infer that it is likely that the directions in $\gamma$ are large together while the ones in $\gamma^c$ take small values. The marginals in $\gamma$ form a cluster of extremal directions for which the relative importance of each direction can be studied via the identification of non-maximal directions $\beta \subsetneq \gamma$ such that $\P( \boldZ \in C_\beta) > 0$. Indeed, if we detect such a subset it means that in the cluster $\gamma$ the directions in $\beta$ are likely to be larger than the ones in $\gamma \setminus \beta$. A deeper interpretation of non-maximal directions is deferred to future work.

\begin{ex}\label{ex:non_maximal_subsets}
	We consider a regularly random variable $X$ in $\R_+$ with tail index $\alpha > 0$ and a vector $\bolda \in \R^d_+$. We assume that the coordinates of $\bolda$ satisfy the inequality $a_1 > a_2 > \ldots > a_d > 0$ and we also assume for simplicity that $\bolda \in \mathbb S^{d-1}_+$. We define the vector $\boldX$ by setting $\boldX =  X \bolda = (a_1 X, \ldots, a_d X) \in \R^d_+$. Then, the vector $\boldX$ is regularly varying with tail index $\alpha$ and a spectral vector given by $\boldTheta = \bolda$ a.s. This means that the direction ${\{1, \ldots, d\}}$ is the only one on which the spectral measure places mass, and it is a maximal one. Hence, the angular vector $\boldZ$ satisfies $\P( \boldZ \in C_{\{1, \ldots, d\}}) > 0$. However, it is possible that the distribution of $\boldZ$ also puts mass on lower-dimensional subsets. Since the Euclidean projection keeps the order of the marginals, the only possible groups of directions are $\{1\}, \{1, 2\}, \{1, 2, 3\}, \ldots, \{1, \ldots, d-1\}$.
	
	We first consider the direction $\{1\}$ and compute the probability that $\boldZ$ belongs to the subset $C_{ \{1\} }$. Following Equation \eqref{eq:proj_beta}, we obtain that
	\[
	\P\big( \boldZ \in C_{ \{1\} } \big)
	=\P\Big( \min _{j \geq 2} \, (Y \Theta_1 - Y \Theta_j) \geq 1 \Big)
	= \P\Big( Y \Theta_1 \geq \max_{j \geq 2} \, Y \Theta_j +1 \Big)\, .
	\]
	Then, we use the relation $\boldTheta = \bolda$ a.s. which entails that
	\[
	\P\big( \boldZ \in C_{ \{1\} } \big)
	= \P\big( Y a_1\geq Y a_2+ 1 \big)
	= \P\big( Y \geq (a_1-a_2)^{-1} \big)
	= (a_1-a_2)^\alpha\, .
	\]
	Hence, the probability that $\boldZ$ belongs to the first axis depends on the difference between the first and the second coordinate of $\bolda$.
	
	More generally, for $1 \leq r \leq d-1$, Equation \eqref{eq:proj_beta_beta^c} implies that $\boldZ$ belongs to the subset $C_{ \{1, \ldots, r\} }$ if and only if
	\[
	\max_{1 \leq i \leq r}\, \sum_{j=1}^r (Y \Theta_j - Y \Theta_i) < 1\, ,
	\quad \text{and} \quad
	\min_{r+1 \leq i \leq d}\,  \sum_{j=1}^r (Y\Theta_j - Y\Theta_i) \geq 1 \, .
	\]
	Thus, the probability that $\boldZ$ belongs to the subset $C_{\{1, \ldots, r\}}$ is equal to
	\begin{align*}
	\P\big( \boldZ \in C_{ \{1, \ldots, r\} } \big)
	&=\P\bigg( \sum_{j=1}^r (Y a_j - Y a_r) < 1, \,  \sum_{j=1}^r (Y\Theta_j - Y\Theta_{r+1}) \geq 1\bigg)\\
	&= \P\big( (\tilde a_r - r a_{r+1})^{-1} \leq Y < (\tilde a_r - a_r)^{-1} \big)
	\quad \text{where } \tilde a_r = a_1 + \ldots + a_r\\
	&= (\tilde a_r - r a_{r+1})^\alpha - (\tilde a_r - r a_r)^\alpha\, .
	\end{align*}
	If we take $\alpha = 1$ for the sake of simplicity, then we obtain $\P ( \boldZ \in C_{ \{1, \ldots, r\} } ) = r (a_r - a_{r+1})$ and thus the probability that $\boldZ$ belongs to the subset $C_{\{1, \ldots, r\}}$ depends only on the distance between $a_r$ and $a_{r+1}$.
	
	This example emphasizes the use of the vector $\boldZ$ on non-maximal directions. It highlights the relative importance of a coordinate regarding the extreme values of a group of directions this coordinate belongs to.
\end{ex}

\section{Numerical results} \label{sec:numerical_results}

This section is devoted to a statistical illustration of sparse regular variation. We highlight how our approach manages to detect sparsity in the tail dependence. We provide a method in order to approximate the probabilities $\P(\boldZ \in C_\beta)$ and apply it to several numerical results.

\subsection{The framework}

We consider an iid sequence of regularly varying random vectors $\boldX_1, \ldots, \boldX_n$ with generic distribution $\boldX \in \RV(\alpha, \boldTheta)$. We set $\boldZ = \pi(Y \boldTheta)$ where $Y$ follows a Pareto($\alpha$) distribution independent of $\boldTheta$. Our aim is to capture the directions $\beta \in {\cal P}_d^*$ such that $\P(\boldZ \in C_\beta ) >0$. Thanks to Proposition \ref{prop:Z_on_specific_subsets} the latter probability is defined through the limit
\begin{equation} \label{eq:conv_cones_stat}
\P(\boldZ \in C_\beta )
= \lim\limits_{t \to \infty} \P(\pi(\boldX / t) \in C_\beta \mid |\boldX| > t)
= \lim\limits_{t \to \infty} \frac{\P(\pi(\boldX / t) \in C_\beta, \, |\boldX| > t)}{\P( |\boldX| > t)} \, .
\end{equation}
The goal is then to approximate this probability with the sample $\boldX_1, \ldots, \boldX_n$. We define the quantity
\begin{equation} \label{eq:def_hat_T}
T_\beta(t) = \frac{ \sum_{j = 1}^n \indic \{ \pi(\boldX_j / t) \in C_\beta, |\boldX_j| > t \} }{ \sum_{j = 1}^n \indic \{|\boldX_j| > t \} }\, , \quad t > 0\, , \quad \beta \in {\cal P}_d^*\, ,
\end{equation}
which corresponds to the proportion of data $\boldX_j$ whose projected vector $\pi(\boldX_j /t)$ belongs to $C_\beta$ among the data whose $\ell^1$-norm is above $t$. Intuitively, the larger the variable $T_\beta(t)$, the more likely the direction $\beta$ gathers extreme values. The Law of Large Numbers then entails the following approximation:
\[
T_\beta(t)
\approx
\frac{\P(\pi(\boldX / t) \in C_\beta, \, |\boldX| > t)}{\P( |\boldX| > t)} 
\approx
\P(\boldZ \in C_\beta )\, ,
\]
where the first approximation holds for $n$ large and the second one for $t$ large. This approximation allows one to classify the directions $\beta$ depending on the nullity or not of the associated quantity $T_\beta(t)$. Actually once $t$ is fixed we can get rid of the denominator in \eqref{eq:def_hat_T} and only focus on $\sum_{j = 1}^n \indic \{ \pi(\boldX_j / t) \in C_\beta, |\boldX_j| > t \}$. The selection of the largest vectors $\boldX_j$ whose norm is above $t$ then boils down to keeping only a proportion, say $k$, of vectors. It is customary in EVT to choose a level $k$ which satisfies $k \to \infty$ and $k/n \to 0$ when $n \to \infty$.

\begin{rem}[The approach proposed by \cite{goix_sabourin_clemencon_17}]
In order to detect anomalies among multivariate extremes, \cite{goix_sabourin_clemencon_17} propose a similar approach with the $\ell^\infty$-norm based on the $\epsilon$-thickened rectangles
\[
R^\epsilon_\beta = \{ \boldx \in \R^d_+ : |\boldx|_\infty > 1, \,x_j > \epsilon \text{ for all } j \in \beta, \, x_j \leq \epsilon \text{ for all } j \in \beta^c \}\, , \quad \beta \in {\cal P}_d^*\, .
\]
Starting from the sample $\boldX_1, \ldots, \boldX_n$ with generic random vector $\boldX = (X^1, \ldots, X^d)$ with marginal distribution $F_1, \ldots, F_d$, the authors define the vectors $\boldV_i = (1 / (1-\hat {F_j} (X_i^j)))_{j = 1, \ldots, d}$ for $i = 1, \ldots, n$, where $\hat {F_j} : x \mapsto \frac{1}{n} \sum_{i=1}^n \indic_{X_i^j < x}$ is the empirical counterpart of $F_j$. This rank transformation provides standardized marginals to the vectors $\boldV_i$. Denoting by $\tilde \boldTheta_\infty$ the nonstandard spectral vector with respect to the $\ell^\infty$-norm and by $C_{\beta, \infty} = \big \{ \boldx \in \mathbb S^d_{+, \infty}, \, x_i > 0 \text{ for } i \in \beta, \, x_i = 0 \text{ for } i \notin \beta \big\}$ the associated subsets, \cite{goix_sabourin_clemencon_17} use the approximation
\[
T_\beta(k, \epsilon)
=
\frac{1}{k} \sum_{i=1}^n \indic_{\boldV_i \in (n/k) R^\epsilon_\beta}
\approx
c \, \P( \tilde \boldTheta_\infty \in C_{\beta, \infty})\, , \quad c > 0\, , 
\]
for $k$ large, and $\epsilon$ close to zero (the ratio $n/k$ plays the role of the large threshold $t$, see \cite{dehaan_fereira}). The authors propose an algorithm called DAMEX whose goal is to identify the subsets $C_{\beta, \infty}$ such that $\P( \tilde \boldTheta_\infty \in C_{\beta, \infty}) > 0$.
\end{rem}

\begin{rem}[On the choice of the norm] \label{rem:choice_norm} After some calculations we observe that if the spectral vectors $\tilde \boldTheta$ and $\tilde \boldTheta_\infty$ correspond to nonstandard regular variation (with $\alpha = 1$), then they satisfy the relation
\begin{equation}\label{eq:rel_theta_theta_tilde}
\P(\tilde \boldTheta \in B) = \frac{ \E \big[|\tilde \boldTheta_\infty| \indic_{ \{\tilde \boldTheta_\infty / |\tilde \boldTheta_\infty| \in B\} } \big] }{ \E[|\tilde \boldTheta_\infty| ] }\, ,
\end{equation}
for all $B \in \mathbb S^{d-1}_+$. Since $\tilde \boldTheta_\infty / |\tilde \boldTheta_\infty| \in C_\beta$ if and only if $\tilde \boldTheta_\infty \in C_{\beta, \infty}$ we obtain the equivalence
\[
\P(\tilde \boldTheta \in C_\beta) > 0 \quad \text{if and only if} \quad \P(\tilde \boldTheta_\infty \in C_{\beta, \infty}) > 0\, .
\]
Hence, the directions in which extremes gather are the same regardless the choice of the norm. This means that after a standardization of the marginals we can compare the performance of our method with the one of \cite{goix_sabourin_clemencon_17}. This is what we do in the first numerical example below.

Note that if $\boldX$ is regularly varying with tail index $1$, then $\boldX^q$ is regularly varying with tail index $1/q$ and Equation \eqref{eq:rel_theta_theta_tilde} implies that the corresponding spectral measures concentrate on the same subsets $C_\beta$ (see also Remark \ref{rem:choice_alpha}).
\end{rem}

\begin{rem}\label{rem:threshold_p}
At the end of the procedure we obtain a group of directions $\beta$ such that $T_\beta(t) > 0$. Since we deal with non-asymptotic data, we obtain a bias which provides a difference between some directions $\beta$ for which $T_\beta(t)$ takes small values while the theoretical quantities $\P(\boldZ \in C_\beta)$ are null. We follow the idea of \cite{goix_sabourin_clemencon_17}, Remark 4, to deal with this issue. We define a threshold value under which the empirical quantities $T_\beta(t)$ are set to 0. We use a threshold of the form $p /|{\cal C}|$, where ${\cal C} = \{ \beta, \, T_\beta(t) > 0\}$ and where the hyperparameter $p \geq 0$ is fixed by the user. It is of course possible to set $p$ to $0$ which boils down to selecting all directions $\beta$ such that $T_\beta(t) > 0$. In this case the number of selected $\beta$ is still much smaller than the total number $2^d-1$. We do not detail more the choice of $p$ and defer this issue to future work.
\end{rem}

Taking this hyperparameter $p$ into account, we are now able to introduce the algorithm used to study the dependence structure of sparsely regularly varying random vectors.

\begin{algorithm}[H]
	\SetAlgoLined
	\KwData{$\boldX_1, \ldots, \boldX_n \in \R^d_+$, $t > 0$, $p \geq 0$}
	\KwResult{A list $\cal C$ of directions $\beta$}
	Compute $\pi(\boldX_j/t)$, $j = 1, \ldots, n$\;
	Assign to each $\pi(\boldX_j/t)$ the subsets $C_\beta$ it belongs to\;
	Compute $T_\beta(t)$\;
	Set to $0$ the $T_\beta(t)$ below the threshold discussed in Remark \ref{rem:threshold_p}\;
	Define ${\cal C} = \{ \beta, \, T_\beta(t) > 0 \}$.
	\caption{Extremal dependence structure of sparsely regularly varying random vectors}
	\label{algo:dependence_structure_sparse_rv}
\end{algorithm}

\subsection{Experimental results}

In this section we consider two different cases of numerical data for which we apply Algorithm \ref{algo:dependence_structure_sparse_rv}. For each case we generate data sets of size $n \in \{10^4, 5 \cdot 10^4, 10^5\}$, we compute the quantities $T_\beta(t)$, and we repeat this procedure over $N =100$ simulations. Regarding the outcome ${\cal C} = \{\beta, \, T_\beta(t) > 0 \}$ of our procedure, two different types of errors could arise. The first one corresponds to the occurrence of a direction $\beta$ while it should not appear theoretically. This error will be called error of type 1. The second type of error corresponds to the absence of a direction $\beta$ while it should appear theoretically. This error will be called error of type 2. The results correspond to the average number of each error among the $N$ simulations. The code can be found at \url{https://github.com/meyernicolas/projection_extremes}.

The purpose of the experiments is to study the procedure given in Algorithm \ref{algo:dependence_structure_sparse_rv} and to see how it manages to detect the sparsity in the extremes. We also analyze the influence of the tail index $\alpha$ by choosing different values for this parameter. This is done by considering the random vector $\boldX^\alpha =(X_1^\alpha, \ldots, X_d^\alpha)$ whose tail index is $1/\alpha$ for $\boldX \in \RV(1, \boldTheta)$.

\begin{rem}[Choice of the parameters]
It is common in EVT to define a level of exceedances $k = n \P(|\boldX| > t)$ and to rather work with $k$ instead of $t$. For our simulations, we choose $k = \sqrt n$, following \cite{goix_sabourin_clemencon_17}, who also suggest choosing $\epsilon$ of order $k^{-1/4}$, that is, of order $n^{-1/8}$. This choice of $\epsilon$ is based on theoretical results (\cite{goix_sabourin_clemencon_17}, Theorem 1), but the authors then advise to rather choose $\epsilon = 0.01$ which gives better results on their simulations. In order to have a large scale of comparison we consider $\epsilon \in \{0.05, 0.1, 0.5\}$. Finally we consider $p = 0.3$ which is larger than the value chosen in \cite{goix_sabourin_clemencon_17} but leads to better results for both methods.
\end{rem}

\begin{rem}\label{rem:choice_alpha}
	Assume that $\pi(\boldX/t) \in C_\beta$. This implies that for all $j \in \beta$ we have $X_j > \max_{i \in \beta^c} X_i$ and $|\boldX_\beta| - \max_{i \in \beta^c} X_i > t$, see \eqref{eq:proj_beta_beta^c}. Hence for $k \in \beta$ and $q \geq 1$ we obtain
	\begin{align*}
	|\boldX^q_\beta| - \max_{i \in \beta^c} X_i^q
	= X_k^q - (\max_{i \in \beta^c} X_i)^q + |\boldX^q_{\beta \setminus \{k\}}|
	&= ( X_k - (\max_{i \in \beta^c} X_i) ) \sum_{l=0}^q X_k^l (\max_{i \in \beta^c} X_i)^{q-l-1} + |\boldX^q_{\beta \setminus \{k\}}|\\
	&\geq X_k - (\max_{i \in \beta^c} X_i) + |\boldX_{\beta \setminus \{k\}}|
	=  |\boldX_\beta| - (\max_{i \in \beta^c} X_i)
	\geq t \, ,
	\end{align*}
	as soon as we assume that all marginals satisfy $X_j > 1$. The relation $|\boldX^q_\beta| - \max_{i \in \beta^c} X_i^q$ implies that $\pi(\boldX^q/t)_{\beta^c} = 0$, see \eqref{eq:proj_beta}. In other words, it means that $\pi(\boldX^q/t)$ belongs to $C_\gamma$ for $\gamma \subset \beta$.
	
	In particular, if the spectral measure of $\boldX$ only concentrates on the axis, these directions can be easily detected with the study of $\pi(\boldX^q/t)$, $q \geq 1$, than with the one of $\pi(\boldX/t)$.
\end{rem}

\paragraph{Asymptotic independence}

We consider an iid sequence of random vectors $\boldN_1, \ldots, \boldN_n$ in $\R^{40}$ with generic random vector $\boldN$ whose distribution is a multivariate Gaussian distribution with all univariate marginals equal to ${\cal N}(0,1)$ and the correlations less than $1$: $\E[N^i N^j] < 1$ for all $1 \leq i \neq j \leq d$. We transform the marginals with a rank transform which consists in considering the vectors $\boldX_1, \ldots, \boldX_n$ such that the marginals $X^j_i$ of $\boldX_i = (X_i^1, \ldots, X_i^d)$ are defined as
\[
X_i^j = \frac{1}{1 - \hat{F_j}(N^j_i)}\, , \quad 1 \leq j \leq d\, ,
\]
where $\hat{F_j}$ is the empirical version of the cumulative distribution function of $N_j \sim {\cal N}(0,1)$. This provides a sample of regularly varying random vectors $\boldX_1, \ldots, \boldX_n$ and the assumption on the correlation leads to asymptotic independence, i.e. $\P(\boldTheta \in C_\beta) = \P(\boldZ \in C_\beta) = 1/d$ for all $\beta$ such that $|\beta| = 1$ (\cite{sibuya_1960}). The aim of our procedure is to recover these $40$ directions among the $2^{40} -1 \approx 10^{12}$ ones.

Regarding the multivariate Gaussian random vectors $\boldN_1, \ldots, \boldN_n$, the simulation of these vectors depends only on their covariance matrix. We proceed as follows. We generate a matrix $\boldSigma'$ with entries $\sigma'_{i,j}$ following independent uniform distributions on $(-1,1)$. Then, we define the matrix $\boldSigma$ as
\[
\boldSigma = \Diag({\sigma'}_{1,1}^{-1/2}, \ldots, {\sigma'}_{d,d}^{-1/2}) \cdot \boldSigma'^T \cdot \boldSigma' \cdot \Diag({\sigma'}_{1,1}^{-1/2}, \ldots, {\sigma'}_{d,d}^{-1/2})\, ,
\]
where $\Diag(\sigma_{1,1}^{-1/2}, \ldots, \sigma_{d,d}^{-1/2})$ denotes the diagonal matrix of ${\cal M}_d(\R)$ whose diagonal is given by the vector $({\sigma'}_{1,1}^{-1/2}, \ldots, {\sigma'}_{d,d}^{-1/2})$. This provides a covariance matrix with diagonal entries equal to $1$ and off-diagonal entries less than $1$. A given matrix $\boldSigma$ provides then a dependence structure for $\boldN$ and thus for $\boldX$. We generate $N_{\text{model}} = 20$ different matrices $\boldSigma$ and for each of these dependence structures we generate $N=100$ sample $\boldN_1, \ldots, \boldN_n$. We summarize in Table \ref{fig:results_asymptotic_independence_T1} and in Table \ref{fig:results_asymptotic_independence_T2} the two types of errors averaged among the $N \cdot N_{\text{model}} = 2000$ simulations.

In this case, the standard spectral vector $\boldTheta$ and the non-standard one $\tilde \boldTheta$ coincide. Hence it is possible to compare $\boldTheta$ and $\tilde \boldTheta_\infty$ (see Remark \ref{rem:choice_norm}) which is done by computing the quantities $T_\beta(k, \epsilon)$ as well as the two types of errors for the DAMEX algorithm. We only study the effect of $\alpha$ for our approach since the DAMEX algorithm is not sensitive to marginals.

\begin{table}[!h]
	\begin{center}
	\begin{tabular}{c||c|c|c|c|c|c}
		Errors of & Euclidean & Euclidean & Euclidean & DAMEX& DAMEX& DAMEX\\
		 Type 1 & proj. $\alpha = 1$ & proj. $\alpha = 1/2$ & proj. $\alpha = 2$ &$\epsilon = 0.05$ & $\epsilon = 0.1$ & $\epsilon = 0.5$ \\
		\hline
		$n = 10^4$ & 22.62 & 21.76 & 2.50 & 3034.70 & 2899.05 & 987.63 \\
		\hline
		$n = 5 \cdot 10^4$ & 19.43 & 6.49 & 69.9 & 6972.52 & 4646.43 & 271.87 \\
		\hline
		$n = 10^5$ & 1.83 & 0.65 & 99.79 & 8401.21 & 4813.46 & 235.80 \\
	\end{tabular}
\caption{Average number of errors of type 1 in an asymptotic independence case ($d=40$).}
\label{fig:results_asymptotic_independence_T1}
\end{center}
\end{table}   

\begin{table}[!h]
	\begin{center}
		\begin{tabular}{c||c|c|c|c|c|c}
		Errors of & Euclidean & Euclidean & Euclidean & DAMEX & DAMEX & DAMEX\\
		Type 2 & proj. $\alpha=1$ & proj. $\alpha=1/2$ & proj. $\alpha=2$ & $\epsilon = 0.05$ & $\epsilon = 0.1$ & $\epsilon = 0.5$ \\
		\hline
		$n = 10^4$ & 0.07 & 0.02 & 40.00 & 39.43 & 13.76 & 0.00 \\
		\hline
		$n = 5 \cdot 10^4$ & 0.00 & 0.00 & 4.89 & 3.69 & 0.01 & 0.00 \\
		\hline
		$n = 10^5$ & 0.00 & 0.00 & 0.41 & 0.07 & 0.00 & 0.00 \\
	\end{tabular}
	\caption{Average number of errors of type 2 in an asymptotic independence case ($d=40$).}
	\label{fig:results_asymptotic_independence_T2}
	\end{center}
\end{table}

For the Euclidean projection with $\alpha = 1$ we observe that our algorithm manages to capture almost all $d=40$ directions regardless the value of $n$, and the error of type 2 decreases when $n$ increases. On the other hand, our algorithm still captures some extra-directions, especially for $n = 10^4$ and $n = 5 \cdot 10^4$. This may be a consequence of the choice of $p$ in Remark \ref{rem:threshold_p} which is probably too high and for which a deeper study should be conducted. The error of type 1 is then much lower for $n =10^5$. We observe that for $\alpha=1/2$ we obtain better results, while the number of errors is higher for $\alpha =2$. Since the extremal directions are in this case only one-dimensional, this confirms numerically the observations of Remark \ref{rem:choice_alpha}.

Regarding the DAMEX algorithm, a large $\epsilon$ leads theoretically to more mass assigned on the axis. This explains why in our simulations choosing a large $\epsilon$ reduces the error of type 2. With $\epsilon = 0.5$ the algorithm manages to capture all the $d=40$ axes, however, the error of type 1 is quite large regardless the choice of $n$. Hence it seems that our procedure leads to the best compromise between both types of errors.

\paragraph{A dependent case}

We now consider a dependent case where extremes occur on lower-dimensional directions. In order to include dependence we recall that a vector $\boldP(k) =(P_1, P_1 + P_2, \ldots, P_1 + P_k) \in \R^k_+$ with $P_j$ following Pareto($\alpha_j$), $\alpha_1 < \alpha_j$ for all $ 2 \leq j \leq k$, is regularly varying with tail index $\alpha_1$ and spectral vector $\boldTheta = (1/k,\ldots, 1/k)$. For our simulations we consider $s_1=10$ independent copies $\boldP_1, \ldots, \boldP_{s_1}$ of $\boldP(2) \in \R^2_+$ with $\alpha_1 = 1$ and $\alpha_2 = 2$ and $s_2=10$ independent copies $\boldR_1, \ldots, \boldR_{s_2}$  of $\boldP(3) \in \R^3_+$ with $\alpha_1 = 1$ and $\alpha_2 = \alpha_3 = 2$. We aggregate these vectors and form a vector $\boldX$ in $\R^{50}_+$ which is then regularly varying with a discrete spectral measure placing mass on the points $(\bolde_j + \bolde_{j+1})/2$ for $j = 1, 3, \ldots, 17, 19$ and on the points $(\bolde_j + \bolde_{j+1} + \bolde_{j+2}) / 3$ for $j = 21, 24, \ldots, 45, 48$. Besides, as discussed after Proposition \ref{prop:dependence_Z_Y}, in this case the angular vector $\boldZ$ and the spectral vector $\boldTheta$ are equal almost surely. Our aim is to recover the $s_1 = 10$ two-dimensional directions $(\bolde_j + \bolde_{j+1})/2$ for $j = 1, 3, \ldots, 17, 19$ and also the $s_2 =10$ three-dimensional directions $(\bolde_j + \bolde_{j+1} + \bolde_{j+2}) / 3$ for $j = 21, 24, \ldots, 45, 48$ based on a sample of iid random vectors $\boldX_1, \ldots, \boldX_n$ with the same distribution as $\boldX$. Hence we would like to recover $s = s_1 + s_2 = 20$ directions among the $2^{50}-1 \approx 10^{15}$ directions.

\begin{table}[!h]
	\begin{center}
		\begin{tabular}{c||c|c|c|c|c|c}
			$ $ & \multicolumn{3}{c|}{Errors of type 1} & \multicolumn{3}{c}{Errors of type 2} \\
			\hline
			$ $ & Eucl. proj & Eucl. proj & Eucl. proj & Eucl. proj & Eucl. proj &Eucl. proj \\
			$ $ & $\alpha=1$ & $\alpha=1/2$ & $\alpha=2$ & $\alpha=1$ & $\alpha=1/2$ & $\alpha=2$ \\
			\hline
			$n = 10^4$ & 8.22 & 0.01 & 26.23 & 0.76 & 0.74 & 7.78 \\
			\hline
			$n = 5 \cdot 10^4$ & 0.32 & 0.00 & 59.31 & 0.04 & 0.08 & 1.11\\
			\hline
			$n = 10^5$ & 0.04 & 0.00 & 78.97 & 0.03 & 0.01 & 0.35\\
		\end{tabular}
		\caption{Average number of errors of type 1 and 2 in a dependent case ($d = 50$).}
		\label{fig:results_dependence_T1}
	\end{center}
\end{table}

As for asymptotic independence, we remark that the number of errors decreases when $n$ increases for almost all cases. For $\alpha = 1$ and $\alpha = 1/2$ our algorithm is not only able to detect all the $s = 20$ directions on which the distribution of $\boldZ$ puts mass, but it also does not identify some extra directions. These results are all the more accurate since the identification of the $s = 20$ directions is done among a very large number of directions, in this case $2^{50} - 1 \approx 10^{15}$. Besides, we obtain very low errors of both types for $\alpha =1$ and $\alpha =1/2$. The better results we obtain for the error of type 1 in the case $\alpha = 1/2$ compared to the ones in the case $\alpha = 1$ can however not be explained by Remark \ref{rem:choice_alpha}. We defer this question to future work.

\subsection{Sparse regular variation and non-maximal directions}\label{subsec:non_maximal_numerical}

In this section we illustrate some interpretation of the vector $\boldZ$ regarding extremal directions in non-maximal directions (see the discussion in Section \ref{subsec:maximal_directions}). We consider a vector $\bolda \in \mathbb S^{r-1}_+$ as in Example \ref{ex:non_maximal_subsets} and a Pareto$(\alpha)$ distributed random variables $P$ and define $P \bolda = (a_1 P, \ldots, a_r P) \in \RV(\alpha, \bolda)$. Then, combining this device with the one of the dependent case we consider $\boldP = (a_1 P, a_2 P + P_2, \dots, a_r P + P_r) \in \RV(\alpha, \bolda)$ where $P_2, \ldots, P_r$ are iid random variables following a Pareto distribution with parameter $\alpha' > \alpha$. Hence the degenerate spectral vector $\boldTheta = \bolda$ only places mass in the direction ${ \{1, \ldots, r\} }$, which is thus maximal, while the vector $\boldZ$ places mass in all non-maximal directions $\{1\}, \{1, 2\}, \{1, 2, 3\}, \ldots, \{1, \ldots, d-1\}$, see Example \ref{ex:non_maximal_subsets}. In our simulations, we choose $\alpha =1$, $\alpha' = 2$, and $r=3$, and we consider a vector $\bolda = (7,6,4)/|(7,6,4)|$. We then aggregate $s=20$ iid copies $\boldP_k$ of the vector $\boldP$ and obtain $\boldX = (\boldP_1, \ldots, \boldP_s) \in \RV(1, \boldTheta)$ with $\boldTheta$ placing mass on the three-dimensional maximal directions ${ \{j, j+1, j+2\} }$ for $j \in J = \{1, 4, 7, \ldots, 58\}$ while the angular vector $\boldZ$ places mass on the aforementioned maximal directions but also on $s=20$ two-dimensional and $s=20$ one-dimensional directions.

The aim of the simulations is to see to what extent our procedure manages to recover the $40$ non-maximal directions aforementioned. The columns of Table \ref{fig:results_proportional} gives the averaged number of the directions that have been recovered by our algorithm, depending on there size. Recall that for each type of directions the theoretical number of these directions that should appear is $s=20$. Finally, the last column deals with the number of directions that should not appear theoretically. All the results are averaged among the $N=100$ simulations.

\begin{table}[!h]
	\begin{center}
		\begin{tabular}{c||c|c|c|c}
			& Three-dimensional & Two-dimensional & One-dimensional & Other\\
			& directions & directions & directions & directions\\
			\hline
			$n = 10^4$ &13.16 & 12.28 & 17.92 & 14.96\\
			\hline
			$n = 5 \cdot 10^4$ & 18.40 & 18.04 & 19.91 & 17.35\\
			\hline
			$n = 10^5$ & 17.95 & 17.39 & 19.92 & 0.68\\
		\end{tabular}
		\caption{Average number of directions recovered by Algorithm \ref{algo:dependence_structure_sparse_rv} ($d=60$).}
		\label{fig:results_proportional}
	\end{center}
\end{table}

For $n= 10^4$ we observe that the procedure manages to identify most of the one-dimensional directions, while the average number of two-dimensional directions is quite smaller than the theoretical one. The same arise for the maximal directions for which we only manage to recover two third of the theoretical ones. We also obtain a non-negligible number of extra-directions which should not be identified. For $n = 5 \cdot 10^4$, the three types of directions are quite well recovered by our algorithm, with once again very good result for the one-dimensional ones. The number of extra-directions is still relatively high. For $n= 10^5$, we keep a high level of accuracy regarding the identification of the three types of directions while the number of extra-directions drastically decreases.

This example highlights the relevance of our approach to identify clusters of directions that are simultaneously large but also to study the relative importance of the coordinates in a given cluster. This second aspect provides a deeper interpretation of $\boldZ$ in terms of extremes.

\section{Conclusion} \label{sec:conclusion}

The notion of sparse regular variation that is introduced in this paper is a way to tackle the issues that arise in the study of tail dependence with the standard concept of regular variation. Replacing the self-normalized vector $\boldX/|\boldX|$ by the projected one $\pi(\boldX / t)$ allows us to capture the extremal directions of $\boldX$. Our main result is the equivalence between both concepts of regular variation under mild assumptions.

Regarding extremes values, the vector $\boldZ$ enjoys many useful properties. This vector is sparser than the spectral vector $\boldTheta$ which entails that it seems more suitable to identify extremal directions, especially in high dimensions. Indeed, large events often appear due to a simultaneous extreme behavior of a small number of coordinates. This similarity between extreme values and the vector $\boldZ$ appears even more with the subsets $C_\beta$ which highlight the tail dependence of $\boldX$. Proposition \ref{prop:Z_on_specific_subsets} provides a natural way to capture the behavior of $\boldZ$ on these subsets and proves that the Euclidean projection manages to circumvent the weak convergence's issue which arises in the standard definition of regularly varying random vectors.

Practically speaking, Section \ref{sec:numerical_results} illustrates the advantages of our approach for the study of large events. First, using the Euclidean projection allows to study tail dependence without introducing any hyperparameter. On the contrary, the introduction of $\epsilon$-thickened rectangles in \cite{goix_sabourin_clemencon_17} requires to identify a suitable $\epsilon$. Hence, our procedure reduces the algorithmic complexity by avoiding running the given code for different $\epsilon$. Since the projection can be computed in expected linear time, the study of extreme events can then be done in reasonable time in high dimensions. More generally, the numerical results we obtain highlight the efficiency of our method to detect extremal directions. The future work should address the question of the threshold $t$, or equivalently the level $k$, and the bias issue introduced in Remark \ref{rem:threshold_p}. Moreover, a comparison between $\boldZ$ and $\boldTheta$ on non-maximal is also a crucial point to tackle. To this end, a deeper study of the behavior of $\boldZ$ on these kind of subsets should be conducted.

\appendix

\section{Appendix A: Algorithms}\label{sec:appendix_algo}

We introduce here two algorithms which compute the Euclidean projection $\pi_z(\boldv)$ given $\boldv \in \R^d_+$ and $z > 0$.

\begin{algorithm}
	\SetAlgoLined
	\KwData{A vector $\boldv \in \R^d_+$ and a scalar $z > 0$}
	\KwResult{The projected vector $\boldw = \pi_z(\boldv)$}
	Sort $\boldv$ in $\boldmu$ : $\mu_1 \geq \ldots \geq \mu_d$\;
	Find $\rho_{\boldv, z}$ as in \eqref{def:rho}\;
	Define $\lambda_{\boldv, z} = \frac{1}{\rho_{\boldv, z}} \left( \sum_{r=1}^{\rho_{\boldv, z}} \mu_j - z \right)$\;
	\textbf{Output:} $\boldw$ s.t. $w_i = \max(v_i - \lambda_{\boldv, z}, 0)$.
	\caption{Euclidean projection onto the simplex.}
	\label{algo:projection_onto_simplex_almost_linear}
\end{algorithm}
$ $

Algorithm \ref{algo:projection_onto_simplex_almost_linear} emphasizes the number of positive coordinates $\rho_{\boldv, z}$ of the projected vector $\pi_z(\boldv)$:
\begin{equation} \label{def:rho}
\rho_{\boldv, z} = \max \Big\{ j =1, \ldots, d : \mu_j - \frac{1}{j} \big( \sum_{r \leq j} \mu_r - z \big) > 0 \Big\}\, ,
\end{equation}
where $\mu_1 \geq \ldots \geq \mu_d$ denote the order coordinates of $\boldv$, see \cite{duchi_et_al_2008}, Lemma 2. In other words a coordinate $j$ satisfies $\pi_z(\boldv)_j > 0$ if and only if
\begin{equation}\label{eq:rho_coord_positive}
v_j - \frac{1}{\sum_{k\leq d} \indic_{v_k \geq v_j} }
\Big(
\sum_{k=1}^d v_k \indic_{v_k \geq v_j} -z
\Big) >0\, . 
\end{equation}
The integer $\rho_{\boldv, z}$ corresponds to the $\ell^0$-norm of $\pi_z(\boldv)$ and thus informs on the sparsity of this projected vector. For $z=1$ we simply write $\rho_\boldv$.

A major remark is that Algorithm \ref{algo:projection_onto_simplex_almost_linear} allows one to compute $\pi_z(\boldv)$ as soon as we know the set $\beta$ of positive coordinates of this vector. Indeed if $\beta = \{j \leq d : \pi_z(\boldv)_j > 0\}$, then $\pi_z(\boldv)_j = v_j - (|\boldv_\beta|-z)/ |\beta|$ for $j \in \beta$.

Algorithm \ref{algo:linear_projection_onto_simplex} is a expected linear-time algorithm based on a median-search procedure.

\begin{algorithm}
	\SetAlgoLined
	\KwData{A vector $\boldv \in \R^d_+$ and a scalar $z > 0$}
	\KwResult{The projected vector $\boldw = \pi(\boldv)$}
	Initialize $U = \{1,\ldots,d\}$, $s = 0$, $\rho = 0$\;
	\While{$U \neq \emptyset$}{
		Pick $k \in U$ at random\;
		Partition $U$: $G = \{j \in U : v_j \geq v_k\}$ and $L = \{j \in U : v_j < v_k\}$\;
		Calculate $\Delta \rho = |G|, \quad \Delta s = \sum_{j \in G} v_j$\;
		\eIf{$(s + \Delta s) - (\rho + \Delta \rho)v_k < z$}{
			$s = s +\Delta s$\;
			$\rho = \rho + \Delta \rho$\;
			$U \leftarrow L$\;
		}{
			$U \leftarrow G \setminus \{k\}$\;
		}
	}
	Set $\eta = (s-z)/\rho$.\;
	\textbf{Output:} $\boldw$ s.t. $w_i = v_i - \eta$.
	\caption{Expected linear time projection onto the positive sphere $\mathbb S^{d-1}_+(z)$.}
	\label{algo:linear_projection_onto_simplex}
\end{algorithm}
$ $

\section{Appendix B: Continuation of Example \ref{ex:uniform}}\label{app:calculations}
	Recall that $\Theta_1$ is uniformly distributed on $(0,1)$ and that the distribution of $Z_1$ is given by $\delta_0/4 + \delta_1/4 + U(0,1)/2$.

	We first check that Equation \eqref{eq:Z_terms_of_Theta} holds for $\beta = \{1\}$ and $\beta = \{1, 2\}$. For $\beta = \{1\}$ we have $G_\beta(x,0) = \P(Z_1 = 1) = 1/4$ while $\E[(\Theta_1 - \Theta_2)_+] = \int_0^{1/2} (1- 2 u) \, \mathrm d u = 1/4$. For $\beta = \{1,2\}$, consider $x_1, x_2$ in $(0,1)$ such that $x_1 + x_2 < 1$. On the one hand the quantity $G_\beta(x_1,x_2)$ corresponds to $\P(Z_1 > x_1, Z_2 > x_2) = \P(x_1 < Z_1 < 1-x_2) = 1/2-(x_1+x_2)/2$. On the other hand, if we assume that $x_1, x_2 < 1/2$, then a similar calculation as in Example \ref{ex:uniform} leads to
	\begin{align*}
	\E\Big[ \Big( 1 - \Big(\frac{2\Theta_1-1}{2x_1-1}\Big)_+ &\vee \Big(\frac{2\Theta_2-1}{2x_2-1}\Big)_+ \Big)_+\Big]\\
	&= \E\Big[ \Big( 1 - \Big(\frac{1-2\Theta_1}{1-2x_1}\Big)_+ \vee \Big(\frac{2\Theta_1-1}{1-2x_2}\Big)_+ \Big)_+\Big]\\
	&= \int_0^{1/2} \Big( 1 - \frac{1-2u}{1-2x_1}\Big)_+ \, \mathrm d u + \int_{1/2}^1 \Big( 1 - \frac{2u-1}{1-2x_1}\Big)_+ \, \mathrm d u\\
	&= \int_{x_1}^{1/2} \frac{u-x_1}{1/2-x_1} \, \mathrm d u + \int_{1/2}^{1-x_2} \frac{1-x_2-u}{1/2-x_2} \, \mathrm d u\\
	&= \frac{[1/4 - x_1 + x_1^2]/2}{1/2-x_1} + \frac{[(1-x_2)^2 - (1-x_2) + 1/4]/2}{1/2-x_2}\\
	& = \frac{1/2-x_1}{2} + \frac{1/2-x_2}{2}\\
	&= 1/2 - (x_1+x_2)/2\, .
	\end{align*}
	This proves that Equation \eqref{eq:Z_terms_of_Theta} holds true.
	
	Moving on to Equation \eqref{eq:relation_Theta_Z}, we first consider $\beta = \{1, 2\}$ and $\boldx = (x_1, x_2)$ with $x_1, x_2 > 0$. Since $\Theta$ is uniformly distributed on $(0,1)$ we obtain that $\P(\boldTheta \in A_\boldx) = \P(x_1 < \Theta < 1-x_2) = 1-(x_1 + x_2)$. On the other hand we already proved that $\P(\boldZ \in A_\boldx) = 1/2-(x_1+x_2)/2$. We now have to compute the differential of $H_{\beta, \beta}(\boldx) = G_\beta(\boldx)$. For $\epsilon > 0$ we obtain
	\begin{align*}
	G_\beta(x_1 + \epsilon, x_2) - G_\beta(x_1, x_2)
	&= - \P(x_1 < Z_1 \leq x_1 +\epsilon, Z_2 > x_2)\\
	&= - \P(x_1 < Z_1 \leq x_1 +\epsilon, Z_1 < 1 - x_2)\\
	&= - \P(x_1 < Z_1 \leq x_1 +\epsilon) \quad \text{ for $\epsilon$ small enough}\\
	&= -\epsilon/2\, .
	\end{align*}
	Thus the differential of $G_\beta$ satisfies the relation 
	\[
	d G_\beta(x_1,x_2) \cdot (x_1-1/2, x_2 - 1/2) = (1/2 - x_1 + 1/2 - x_2)/2 =1/2 - (x_1 + x_2)/2\, .
	\]
	Hence the relation \eqref{eq:relation_Theta_Z} is satisfied. Now for $\beta = \{1\}$ and for $x \in (0,1)$ we have $\P(\Theta_1 > x) = 1-x$ and $\P(Z_1 > x) = \P(Z_1 > x, Z_2 = 0) + \P(Z_1 > x, Z_2 > 0) = 1/4 + (1-x)/2$. Hence we have to prove that the sum in \eqref{eq:relation_Theta_Z} adds up to $-1/4 + (1-x)/2$. For $\gamma = \beta$ we obtain that $H_{\beta, \beta}(x, u) = \P(Z_1+Z_2 > x, Z_2 \leq u) = \P(Z_2 \leq u)$ which is constant with respect to $x$ and satisfies
	\[
	H_{\beta, \beta}(x, \epsilon) - H_{\beta, \beta}(x, 0)= \P( 0 < Z_2 \leq \epsilon) = \epsilon/2\, .
	\]
	This implies that $\mathrm d H_{\beta, \beta}(x,0) (x-1, -1/2) = -1/4$. For $\gamma = \{1, 2\}$ we have the relation $H_{\beta, \gamma}(x, u) = \P(Z_1 > x, Z_2 > u) = G_\gamma(x,u)$ which has already been studied above with the case $\beta = \{1,2\}$. This found that $\mathrm d G_\gamma(x,u) = (-1/2, -1/2)$, and thus $\mathrm d G_\gamma(x,0) (x-1/2, -1/2) = (1-x)/2$. Hence we proved that
	\[
	\mathrm d H_{\beta, \beta}(x, 0) \cdot (x-1,-1/2) + \mathrm d H_{\beta, \{1,2\}}(x, 0) \cdot (x-1/2,-1/2) = -1/4 + (1-x)/2\, .
	\]

\section{Appendix C: Proofs} \label{sec:appendix_proofs}

\subsection{Some results on the projection}
We start with this section with three Lemmas which gather some useful properties satisfied by the projection $\pi$.

\begin{lem}[Iteration of the projection] \label{lem:iteration_projection}
	If $0 < z \leq z'$, then $\pi_z \circ \pi_{z'} = \pi_z$.
\end{lem}

\begin{proof}[Proof of Lemma \ref{lem:iteration_projection}]
	The proof of this result relies on the relation $\pi_z(\boldv) = z \pi(\boldv/z)$  and on the characterization \eqref{eq:rho_coord_positive}.
	
	First we simplify the problem via the equivalences
	\begin{align*}
	\forall \, 0 < z \leq z', \, \forall \, \boldv \in \R^d_+, \, \pi_z( \pi_{z'} (\boldv) ) = \pi_z( \boldv )
	\iff &\forall \, 0 < z \leq z', \, \forall \, \boldv \in \R^d_+, \, z \pi( z^{-1} \pi_{z'} (\boldv) ) = z \pi( \boldv  / z)\\
	\iff &\forall \, 0 < z \leq z', \, \forall \, \boldv \in \R^d_+, \, \pi( z' z^{-1} \pi (\boldv / z') ) = \pi( \boldv  / z)\\
	\iff &\forall \, a \geq 1, \, \forall \, \boldu \in \R^d_+, \, \pi (a \pi (\boldu)) = \pi( a \boldu)\, .
	\end{align*}
	So we fix $a \geq 1$ and $\boldu \in \R^d_+$ and we prove this last equality by proving first that $\rho_{a \pi(\boldu)} = \rho_{a \boldu}$ and second that the positive coordinates of both vectors $\pi(a \pi(\boldu))$ and $\pi(a \boldu)$ coincide. \\
	
	\noindent STEP 1: We prove that $\rho_{a \pi(\boldu)} = \rho_{a \boldu}$.
	
	The characterization \eqref{eq:rho_coord_positive} entails that a coordinate $j$ satisfies $\pi(a \pi(\boldu))_j > 0$ if and only if
	\begin{equation}\label{eq:rho_equiv_expr}
	a \pi(\boldu)_j > \frac{1}{\sum_{k\leq d} \indic_{a \pi(\boldu)_k \geq a \pi(\boldu)_j} }
	\Big(
	\sum_{k=1}^d a \pi(\boldu)_k \indic_{a \pi(\boldu)_k \geq a \pi(\boldu)_j} -1
	\Big)\, .
	\end{equation}
	Since $a \geq 1$ this assumption holds only if $\pi(\boldu)_j = u - \lambda_\boldu > 0$. Hence, since $\pi$ preserves the order of the coordinates, we obtain that \eqref{eq:rho_equiv_expr} is equivalent to
	\begin{equation}\label{eq:rho_equiv_expr_2}
	a (u_j -\lambda_\boldu) > \frac{1}{\sum_{k\leq d} \indic_{a u_k \geq a u_j} }
	\Big(
	\sum_{k=1}^d a (u_k - \lambda_\boldu) \indic_{a u_k \geq a u_j} -1
	\Big)\, .
	\end{equation}
	The terms with $\lambda_\boldu$ vanish and \eqref{eq:rho_equiv_expr_2} is then equivalent to $\pi(a\boldu)_j > 0$. Hence the vectors $\pi( a \pi(\boldu) )$ and $\pi( a \pi(\boldu) )$ have the same positive coordinates, i.e. $\rho_{a \pi(\boldu)} = \rho_{a \boldu}$.\\
	
	\noindent STEP 2: We prove that $\pi( a \pi (\boldu) ) = \pi(a \boldu)$.
	
	Let us denote by $\gamma$ the set of coordinates $j$ such that $\pi( a \pi (\boldu) )_j > 0$. STEP 1 ensures that it corresponds to the set of coordinates $j$ such that $\pi(a \boldu)_j > 0$. We prove that these both components coincide. Recall that for $k \in \gamma$ we have $\pi(\boldu)_k = u_k - \lambda_\boldu > 0$. Then the result follows from equalities

	\[
	\pi( a \pi (\boldu) )_j
	= a \pi (\boldu)_j  - \frac{|a \pi(\boldu)_\gamma|-1}{|\gamma|}
	= a (u_j - \lambda_\boldu) - \frac{|a \boldu_\gamma| - a\lambda_\boldu |\gamma| -1}{|\gamma|}
	= a u_j - \frac{|a \boldu_\gamma| -1}{|\gamma|}
	= \pi(a\boldu)_j\, ,
	\]
	for $j \in \gamma$.
\end{proof}

In the following lemma we compare the behavior of the vectors $\boldu/|\boldu|$ and $\pi(\boldu/t)$ on the sets $A_\boldx$.

\begin{lem}[Euclidean projection and self-normalization] \label{lem:upper_lower_bound}
	Let $t > 0$, $\epsilon > 0$ small enough, $\gamma \in {\cal P}_d^*$, and $\boldx \in B_+^d(0,1)$. We consider $\boldu \in \R^d_+$ such that $|\boldu|/t \in (1, 1+\epsilon]$ and $\pi(\boldu/t) \in C_\gamma$.
	\begin{enumerate}
		\item If $\boldu / |\boldu| \in A_\boldx$, then $\pi(\boldu/t) \in A_{\boldx - \epsilon/|\gamma|}$. In particular, this holds only for $\{j : x_j >0\} \subset \gamma$.
		\item If $\pi(\boldu/t) \in A_{\boldx(1+\epsilon)}$, then $\boldu / |\boldu| \in A_\boldx$.
	\end{enumerate}
\end{lem}

\begin{proof}[Proof of Lemma \ref{lem:upper_lower_bound}]
	\begin{enumerate}
		\item The assumption $\pi(\boldu/t) \in C_\gamma$ implies that $\lambda_{\boldu/t} = (|\boldu_\gamma|/t-1)/|\gamma|$ so that for any $j \leq d$ we have
		\[
		\pi(\boldu/t)_j
		= \max ( u_j / t - \lambda_{\boldu/t}, 0 )
		\geq \frac{u_j}{t} - \frac{|\boldu_\gamma|/t - 1}{|\gamma|}
		> \frac{u_j}{|\boldu|} - \frac{|\boldu|/t - 1}{|\gamma|}
		\geq \frac{u_j}{|\boldu|} - \frac{\epsilon}{|\gamma|} \, , 
		\]
		where we used that $1<|\boldu|/t \leq 1+\epsilon$. Then the assumption $\boldu/|\boldu| \in A_\boldx$ concludes the proof.
		
		\item If $\pi(\boldu/t) \in A_{\boldx(1+\epsilon)}$, then we have the inequality
		\[
		\frac{u_j}{|\boldu|}
		= \frac{t}{|\boldu|} \frac{u_j}{t}
		\geq \frac{1}{1+\epsilon} \pi(\boldu/t)_j
		\geq x_j\, ,
		\quad 1 \leq j \leq d\, .
		\]
		\end{enumerate}
\end{proof}

For $\gamma\in {\cal P}_d^*$ recall that the function $\phi_\gamma : \mathbb S^{d-1}_+ \to \R^d_+$ is defined by
\[
\phi_\gamma(\boldu)_j
= \left\{
\begin{array}{ll}
&u_j + \frac{|\boldu_{\gamma^c}|}{|\gamma|}\, , \quad j \in \gamma\, ,\\
&u_j + \frac{|\boldu_{\gamma^c \setminus \{j\} }|}{|\gamma|+1} \, , \quad j \in \gamma^c\, .
\end{array}
\right.
\]
Besides, we have defined the quantities $v_{\beta, \, i} = \sum_{j \in \beta} (v_j - v_i)$ and $v_{\beta, \, i, \, +} = \sum_{j \in \beta} (v_j - v_i)_+$ for $\boldv \in \mathbb S^{d-1}_+$, $\beta \in {\cal P}_d^*$, and $1 \leq i \leq d$.

\begin{lem}[Euclidean projection and sparsity] \label{lem:proj_beta_betac}
	Let $\beta \in {\cal P}_d^*$.
	\begin{enumerate}
		\item For $\boldv \in \R^d_+$ we have the equivalences
		\begin{equation} \label{eq:proj_beta}
		\pi(\boldv)_{\beta^c} = 0 \quad \text{if and only if} \quad 1 \leq \min_{i \in \beta^c} v_{\beta, \, i, \, +}\, ,
		\end{equation}
		and
		\begin{equation} \label{eq:proj_beta_beta^c}
		\pi(\boldv)\in C_\beta \quad \text{if and only if} \quad \left\{
		\begin{array}{ll}
		&\max_{i \in \beta} v_{\beta, \, i} < 1\, ,\\
		&\min_{i \in \beta^c} v_{\beta, \, i} \geq 1 \, .
		\end{array}
		\right.
		\end{equation}
		\item For $\boldx \in {\cal X}_\beta$, $\gamma \supset \beta$, $\boldu \in \mathbb S^{d-1}_+$, and $a \geq 1$ we have the equivalence
		\begin{equation}\label{eq:equiv_pi_A_x_beta}
		\pi(a \boldu) \in A_\boldx \cap C_\gamma \quad \text{if and only if} \quad \left\{
		\begin{array}{lll}
		&\phi_\gamma(\boldu)_j \geq \frac{1}{a} \big(x_j + \frac{a - 1}{|\gamma|}\big)\,, \quad j \in \beta\, ,\\
		& \min_{j \in \gamma \setminus\beta} \phi_\gamma(\boldu)_j > \frac{a - 1}{a |\gamma|}\, ,\\
		& \max_{j \in \gamma^c} \phi_\gamma(\boldu)_j \leq \frac{a - 1}{a(|\gamma|+1)}\, ,
		\end{array}
		\right.
		\end{equation}
	\end{enumerate}
\end{lem}

\begin{proof}[Proof of Lemma \ref{lem:proj_beta_betac}]
	\noindent 1. The characterization \eqref{eq:rho_coord_positive} ensures that $\pi(\boldv)_i = 0$ if and only if
	\[
	v_i - \frac{1}{\sum_{k \leq d} \indic_{v_k \geq v_i}} \Big( \sum_{k\leq d} v_k\indic_{v_r \geq v_i} -1 \Big) \leq 0\, ,
	\]
	which can be rephrased as
	\[
	1 \leq \sum_{k=1}^d (v_k - v_i)\indic_{v_k \geq v_i} \, .
	\]
	This proves \eqref{eq:proj_beta}.
	
	For \eqref{eq:proj_beta_beta^c} the assumption $\pi(\boldv) \in C_\beta$ can be rephrased as follows:
	\[
	\forall i \in \beta, \, v_i = \pi(\boldv)_i + \big(|\boldv_\beta| - 1\big) / |\beta|\quad \text{and} \quad \forall i \in \beta^c, \, v_i \leq (|\boldv_\beta| - 1\big) / |\beta| \, .
	\]
	On the one hand, since $\pi(\boldv)_i > 0$ for $i \in \beta$, the first equality is equivalent to $\max_{i \in \beta} \sum_{j \in \beta} (v_j-v_i) < 1$. On the other hand, the second equality is equivalent to $\min_{i \in \beta^c} \sum_{j \in \beta} (v_j-v_i) \geq 1$. This proves \eqref{eq:proj_beta_beta^c}.
	
	\noindent 2. By definition of the projection $\pi$ the equivalence holds:
	\[
	\left\{
	\begin{array}{lll}
	&\pi(a \boldu)_j  \geq x_j \, , \quad j \in \beta\, ,\\
	&\pi(a \boldu)_j  > 0 \, , \quad j \in \gamma \setminus \beta\, ,\\
	&\pi(a \boldu)_j = 0\, , \quad j \in \gamma^c\, ,
	\end{array}
	\right.
	\quad
	\iff
	\quad
	\left\{
	\begin{array}{lll}
	&a u_j  - \frac{|a\boldu_\gamma| - 1}{|\gamma|}\geq x_j \, , \quad j \in \beta\, ,\\
	&a u_j  - \frac{|a\boldu_\gamma| - 1}{|\gamma|} >0 \, , \quad j \in \gamma \setminus \beta\, ,\\
	&a u_j  \leq \frac{|a\boldu_\gamma| - 1}{|\gamma|}\, , \quad j \in \gamma^c\, .
	\end{array}
	\right.
	\]
	Then, we write $a |\boldu_\beta| = a - a |\boldu_{\beta^c}|$ and the former conditions are equivalent to
	\[
	\left\{
	\begin{array}{lll}
	&a u_j  + \frac{a|\boldu_{\gamma^c}|}{|\gamma|} \geq x_j + \frac{a-1}{|\gamma|} \, , \quad j \in \beta\, ,\\
		&a u_j  + \frac{a|\boldu_{\gamma^c}|}{|\gamma|} > \frac{a-1}{|\gamma|} \, , \quad j \in \gamma \setminus \beta\, ,\\
	&a u_j + \frac{a|\boldu_{\gamma^c}|}{|\gamma|}  \leq \frac{a-1}{|\gamma|}\, , \quad j \in \gamma^c\, .
	\end{array}
	\right.
	\]
	We conclude the proof by writing $\boldu_{\beta^c}$ = $|\boldu_{\beta^c \setminus \{j\}}| + u_j$ for $j \in \gamma^c$.
\end{proof}

%

\subsection{Proof of Proposition \ref{prop:dependence_Z_Y}}
This proof is a consequence of Lemma \ref{lem:iteration_projection}. Indeed for $r \geq 1$, $t> 0$, and $A \subset \mathbb{S}^{d-1}_+$ we have
\begin{align*}
\P \big( \pi (\boldX / t ) \in A, \, |\boldX| / t > r \mid |\boldX| > t \big)
&= \P \big( \pi (\boldX / t ) \in A \mid |\boldX| / t > r \big) \frac{\P( |\boldX| > tr)}{\P( |\boldX| > t)}\\
&= \P \big( \pi (r \boldX / (t r) ) \in A \mid |\boldX| / t > r \big) \P( |\boldX| > tr \mid |\boldX| > t)\\
&= \P \big( r \pi_{1/r} (\boldX / (t r) ) \in A \mid |\boldX| > tr \big) \P( |\boldX| > tr \mid |\boldX| > t)\\
&= \P \big( r \pi_{1/r} ( \pi (\boldX/ (t r) ) ) \in A \mid |\boldX| > tr \big) \P( |\boldX| > tr \mid |\boldX| > t)\, ,
\end{align*}
where the last equality results from Lemma \ref{lem:iteration_projection}. Then, when $t \to \infty$ the continuity of $\pi_{1/r}$ and $\pi$ entails that
\[
\P( \boldZ \in A, \, Y > r) =  \P ( r \pi_{1/r} \left(\boldZ \right)  \in A ) \P(Y > r)\, ,
\]
and applying again Lemma \ref{lem:iteration_projection} concludes the proof.
	
\subsection{Proof of Corollary \ref{corr:Z_points_e_beta}}

Let $\bolda \in \mathbb S^{d-1}_+$ such that $\P( \boldZ = \bolda ) > 0$ and define $\beta = \{j : a_j > 0\}$. The goal is to prove that $\bolda = \bolde(\beta)/|\beta|$.

Since $\P(\boldZ =\bolda \mid Y > r) \to \P(\boldZ = \bolda)$ when $r \to 1$, there exists $r_0 > 1$ such that $\P(\boldZ =\bolda \mid Y > r) > 0$ for all $r \in (1, r_0)$. Proposition \ref{prop:dependence_Z_Y} and Lemma \ref{lem:proj_beta_betac} then imply that for all $r < r_0$ we have
\[
0 < \P( \pi(r \boldZ) = \bolda )
= \P \Big( Z_j - \frac{|\boldZ_\beta|}{|\beta|} = \frac{a_j - 1/|\beta|}{r} \text{ for } j \in \beta, \, Z_j - \frac{|\boldZ_\beta|}{|\beta|} \leq \frac{-1}{|\beta| r} \text{ for } j \in \beta^c \Big)
\]
If there exists $j$ such that $a_j \neq 1/|\beta|$, then the quantities $\frac{a_j - 1/|\beta|}{r}$ are all distinct when $r$ varies in $(1, r_0)$. This contradicts the fact that $\boldZ$ has a discrete distribution. Hence we have $a_j = 1/|\beta|$ for all $j \in \beta$.

\subsection{Proof of Theorem \ref{theo:relation_G_Theta}}

\subsubsection{The distribution of $\boldZ$ in terms of $\boldTheta$}
We consider $\beta \in {\cal P}_d^*$ and $\boldx \in {\cal X}_\beta^0$ such that $x_j \neq 1/|\beta|$ for all $j \in \beta$. We define the quantities $B_j = [(|\beta|\Theta_j - |\boldTheta_\beta|) / (|\beta| x_j - 1)]_+$ for $j \in \beta$. Then we obtain that
\begin{align*}
G_\beta(\boldx)
&= \P( \pi(Y \boldTheta)_\beta > \boldx, \pi(Y \boldTheta)_{\beta^c}= \boldo_{\beta^c})\\
&= \P(Y \Theta_j - (Y|\boldTheta_\beta|-1)/|\beta| > x_j \text{ for }  j \in \beta, \,
Y \Theta_j \leq (Y|\boldTheta_\beta|-1) / |\beta| \text{ for }  j \in \beta^c).
\end{align*}
For $j \in \beta_+$ the condition $Y \Theta_j - (Y|\boldTheta_\beta|-1)/|\beta| > x_j$ is equivalent to $B_j > 1/Y$, whereas for $j \in \beta_-$ it is equivalent to $B_j < 1/Y$. Moreover, for $j \in \beta^c$ the condition $Y \Theta_j \leq (Y|\boldTheta_\beta|-1) / |\beta|$ is equivalent to $|\boldTheta_\beta| - |\beta| \Theta_j \geq 1/Y$. All in all we obtain that $G_\beta(\boldx)$ is equal to 
\begin{align*}
\P \big(
Y^{-\alpha} < B_j^\alpha \, \text{ for } j \in \beta_+, \, 
Y^{-\alpha} > B_j^\alpha \, \text{ for } j \in \beta_-, \, 
Y^{-\alpha} \leq (|\boldTheta_\beta| - |\beta| \Theta_j)_+^\alpha\, \text{ for } j \in \beta^c
\big)\, .
\end{align*}
Since the random variable $Y^{-\alpha}$ follows a uniform distribution on $(0,1)$ and is independent of $\boldTheta$ we obtain the desired result:
\begin{align*}
G_\beta(\boldx)
=
\int_0^1
\P \Big(
u < \min _{j \in \beta_+}  B_j^\alpha, \,
u > \max_{j \in \beta_-} B_j^\alpha&,
\, 
u \leq \min_{j \in \beta^c}(|\boldTheta_\beta| - |\beta| \Theta_j)_+^\alpha
\Big)
\, \mathrm d u\\
&= \int_0^1
\P \Big(
\max_{j \in \beta_-} B_j^\alpha
< 
u < \min _{j \in \beta_+} B_j^\alpha
\wedge
\min_{j \in \beta^c}(|\boldTheta_\beta| - |\beta| \Theta_j)_+^\alpha
\Big)
\, \mathrm d u\\
&= \E \Big[ \Big(
1 \wedge
\min _{j \in \beta_+} B_j^\alpha
\wedge
\min_{j \in \beta^c}(|\boldTheta_\beta| - |\beta| \Theta_j)_+^\alpha
- \max_{j \in \beta_-}B_j^\alpha
\Big)_+ \Big]\, .
\end{align*}

\subsubsection{Equivalence of regular variation and sparse regular variation}

The proof of this result is divided into two steps. The first one consists in characterizing regular variation via the convergence of $\P(|\boldX|/t \leq 1+\epsilon, \, \boldX / |\boldX| \in A_\boldx \mid |\boldX| > t)$ when $t \to \infty$ and $\epsilon \to 0$. This result is stated in the following lemma.

\begin{lem} \label{lem:rv_equivalent_condition}
	Let $\boldX$ be a random vector on $\R^d_+$ and $\alpha > 0$. The following assumptions are equivalent.
	\begin{enumerate}[label={(\arabic*)}]
		\item \label{lem:equiv_1} $\boldX$ is regularly varying with tail index $\alpha$.
		\item \label{lem:equiv_2}
		\begin{enumerate}
			\item[a.] $|\boldX|$ is regularly varying with tail index $\alpha$.
			\item[b.] For all $\beta \in {\cal P}_d^*$ and $\lambda_\beta$-almost every $\boldx \in {\cal X}_\beta$ the quantities
		\begin{equation} \label{eq:rv_equivalent_condition_1}
		(\alpha \epsilon)^{-1}
		\limsup_{t \to \infty}
		\P\big( |\boldX| / t \leq 1+\epsilon, \, \boldX / |\boldX| \in A_\boldx \mid |\boldX| > t \big)
		\end{equation}
		and
		\begin{equation} \label{eq:rv_equivalent_condition_2}
		(\alpha \epsilon)^{-1}
		\liminf_{t \to \infty}
		\P\big( |\boldX| / t \leq 1+\epsilon, \, \boldX / |\boldX| \in A_\boldx \mid |\boldX| > t \big)
		\end{equation}
		have a common limit $l(A_\boldx)$ when $\epsilon > 0$ converges to $0$, and the function $\boldx \mapsto l(A_\boldx)$ is continuous at $\boldx$.
		\end{enumerate}
	\end{enumerate}
	In this case, $l$ extends to a unique probability measure on ${\cal B}(\mathbb S^{d-1}_+)$ which coincides with the spectral measure of $\boldX$.
\end{lem}

The second step then consists in proving that under assumption (A) of Theorem \ref{theo:relation_G_Theta} the second assumption of Lemma \ref{lem:rv_equivalent_condition} holds true.

\begin{proof}[Proof of Lemma \ref{lem:rv_equivalent_condition}]
	
	We first prove that \ref{lem:equiv_1} implies \ref{lem:equiv_2}. If $\boldX \in \RV(\alpha, \boldTheta)$, then $|\boldX|$ is regularly varying with index $\alpha$ and Portmanteau's Theorem ensures that
	\begin{equation}\label{eq:conv_spectral_measure_proof_lemma}
	\lim_{t \to \infty}
	\P\big( |\boldX| / t \leq 1+\epsilon, \, \boldX / |\boldX| \in A_\boldx \mid |\boldX| > t \big)
	= \P( Y \leq 1 + \epsilon) \P(\boldTheta \in A_\boldx)
	= (1 - (1+\epsilon)^{-\alpha}) \P(\boldTheta \in A_\boldx)\, ,
	\end{equation}
	for $\epsilon > 0$, $\beta \in {\cal P}_d^*$, and $\boldx \in {\cal X}_\beta$ such that $\P( \boldTheta \in \partial A_\boldx) = 0$. Since $\P( \boldTheta \in \partial A_\boldx) \leq \sum_{i \in \beta} \P( \Theta_i=x_i )$, the set of all $\boldx \in {\cal X}_\beta$ for which the convergence \eqref{eq:conv_spectral_measure_proof_lemma} does not hold is at most countable and thus is $\lambda_\beta$-negligible. After dividing both sides of \eqref{eq:conv_spectral_measure_proof_lemma} by $\alpha \epsilon$ and taking the limit when $\epsilon$ converges to $0$ we obtain the convergence to $\P(\boldTheta \in A_\boldx)$.\\
	
	We now prove that \ref{lem:equiv_2} implies \ref{lem:equiv_1}. We consider $\beta \in {\cal P}_d^*$ and denote by $\bar {\cal X}_\beta$ the set of all $\boldx \in {\cal X}_\beta$ such that the common limit $l(A_\boldx)$ of \eqref{eq:rv_equivalent_condition_1} and \eqref{eq:rv_equivalent_condition_2} exists and is continuous. We also define $\bar {\cal X} = \{\boldx \in B_+^d(0,1) :  (\boldx_\beta, \boldo_{\beta^c}) \in {\cal X}_\beta \text{ for all } \beta \}$.
	
	For $\epsilon > 0$ and $u > 1$ we decompose the interval $(u,\infty)$ as follows:
	\[
	(u,\infty) = \bigsqcup_{k=0}^\infty \big( u(1+\epsilon)^k, u(1+\epsilon)^{k+1} \big]\, ,
	\]
	where $\bigsqcup$ denotes a disjoint union. Then we obtain for $\boldx \in \bar {\cal X}$ and $t > 0$,
	\begin{align}
	\P \big(|\boldX| / t > u, \, \boldX / |\boldX| &\in A_\boldx \mid |\boldX| > t \big) \label{eq:rv_lemma_decomposition}\\
	&= \sum_{k=0}^\infty \P\left( \frac{|\boldX|}{tu(1+\epsilon)^k} \in (1,1+\epsilon],\, \boldX / |\boldX| \in A_\boldx \; \middle| \; |\boldX| > t \right) \nonumber\\
	&=  \sum_{k=0}^\infty \P\left( \frac{|\boldX|}{tu(1+\epsilon)^k} \leq 1+\epsilon,\, \boldX / |\boldX| \in A_\boldx \; \middle| \;  |\boldX| > t u(1+\epsilon)^k \right) \frac{\P( |\boldX| > t u(1+\epsilon)^k )}{\P(|\boldX| > t)}\, .\nonumber
	\end{align}
	Fatou's Lemma together with the fact that $|\boldX|$ is regularly varying with tail index $\alpha$ imply the two following inequalities:
	\begin{align*}
	\liminf_{t \to \infty}
	\P \big(|\boldX| / t > u, \, &\boldX / |\boldX| \in A_\boldx \mid |\boldX| > t \big)\\
	&\geq \sum_{k=0}^\infty  \big(u(1+\epsilon)^k\big)^{-\alpha} \liminf_{t \to \infty} \P\left( \frac{|\boldX|}{tu(1+\epsilon)^k}  \leq 1+\epsilon,\, \boldX / |\boldX| \in A_\boldx \; \middle| \; \frac{|\boldX|}{u(1+\epsilon)^k} > t \right)\\
	&= \sum_{k=0}^\infty \big(u(1+\epsilon)^k\big)^{-\alpha} \liminf_{t \to \infty} \P\big( |\boldX| / t \leq 1+ \epsilon, \,\boldX / |\boldX| \in A_\boldx \mid |\boldX| > t \big)\\
	&= \frac{u^{-\alpha} }{1-(1+\epsilon)^{-\alpha}} \liminf_{t \to \infty} \P\big( |\boldX| / t \leq 1+ \epsilon,\, \boldX / |\boldX| \in A_\boldx \mid |\boldX| > t \big)\, ,
	\end{align*}
	and
	\begin{align*}
	\limsup_{t \to \infty}
	\P \big(|\boldX| / t > u, \, &\boldX / |\boldX| \in A_\boldx \mid |\boldX| > t \big)\\
	&\leq \sum_{k=0}^\infty  \big(u(1+\epsilon)^k\big)^{-\alpha} \limsup_{t \to \infty} \P\left( \frac{|\boldX|}{tu(1+\epsilon)^k}  \leq 1+\epsilon,\, \boldX / |\boldX| \in A_\boldx \; \middle| \; \frac{|\boldX|}{u(1+\epsilon)^k} > t \right)\\
	&= \sum_{k=0}^\infty \big(u(1+\epsilon)^k\big)^{-\alpha} \limsup_{t \to \infty} \P\big( |\boldX| / t \leq 1+ \epsilon,\, \boldX / |\boldX| \in A_\boldx \mid |\boldX| > t \big)\\
	&= \frac{u^{-\alpha} }{1-(1+\epsilon)^{-\alpha}} \limsup_{t \to \infty} \P\big( |\boldX| / t \leq 1+ \epsilon, \,\boldX / |\boldX| \in A_\boldx \mid |\boldX| > t \big)\, .
	\end{align*}
	Finally, we use Equations \eqref{eq:rv_equivalent_condition_1} and \eqref{eq:rv_equivalent_condition_2} and the relation $1-(1+\epsilon)^{-\alpha} \sim \alpha \epsilon$ which entail that
	\begin{align*}
	u^{-\alpha} l(A_\boldx)
	\leq 
	\liminf_{t \to \infty}
	\P \big(|\boldX| / t > u, \, \boldX / |\boldX|& \in A_\boldx \mid |\boldX| > t \big)\\
	&\leq \limsup_{t \to \infty}
	\P \big(|\boldX| / t > u, \, \boldX / |\boldX| \in A_\boldx \mid |\boldX| > t \big)
	\leq u^{-\alpha} l(A_\boldx)\, .
	\end{align*}
	
	This proves that 
	\begin{equation}\label{eq:conv_sets_A_x_beta_proof}
	\lim_{t \to \infty}
	\P \big(|\boldX| / t > u, \, \boldX / |\boldX| \in A \mid |\boldX| > t \big)
	= u^{-\alpha} l( A )\, ,
	\end{equation}
	for all $A = A_\boldx$ such that $\boldx \in \bar{\cal X}$. The convergence also holds true for $A = \mathbb S^{d-1}_+$ (resp. $A = \emptyset$) with $l(\mathbb S^{d-1}_+) = 1$ (resp. $l(\emptyset) = 0$). Then, by inclusion exclusion we obtain the convergence of $\P \big(\boldX / |\boldX| \leq \boldx \mid |\boldX| > t \big)$ for all $\boldx \in \bar {\cal X}$ and we denote by $F(\boldx)$ this limit which is continuous at any point $\boldx \in \bar {\cal X}$. In particular this implies the convergence $\P( X_j / |\boldX| \leq x \mid |\boldX| > t ) \to 1 - l(A_{x \bolde_j}) =: L_j(x)$ for almost every $x \in (0, 1]$. The functions $L_j$ are non-decreasing and continuous at almost every $x \in (0, 1]$, thus we extend it to a right continuous function on $[0,1]$. Then $F$ is continuous from above (see the definition page 177 in \cite{billingsley_1995}). For $\bolda < \boldb \in \bar {\cal X}$ we have the inequality $\P \big(\bolda < \boldX / |\boldX| \leq \boldb \mid |\boldX| > t \big) \geq 0$ which implies that
	\begin{equation}\label{eq:rectangles_inequality}
	\Delta_{(\bolda, \boldb]} F := \sum_{\boldu \in {\cal V}(\bolda, \boldb)} \text{sign}(\boldu) F(\boldu) \geq 0\, ,
	\end{equation}
	where ${\cal V}(\bolda, \boldb)$ denotes all the vertex of $(\bolda, \boldb]$, and $\text{sign}(\boldv) = 1$ if $u_k = a_k$ for an even number of $k$'s and $\text{sign}(\boldv) = -1$ if $u_k = a_k$ for an odd number of $k$'s. Since $F$ is continuous from above, the inequality \eqref{eq:rectangles_inequality} holds true for all $\bolda < \boldb \in B_+^d(0,1)$. Then, \cite{billingsley_1995} Theorem 12.5 ensures that there exists a measure $S$ on $\R^d$ such that $S((\bolda, \boldb]) = \Delta_{(\bolda, \boldb]} F$. It is straightforward to see that $S$ is a probability measure with support in $\mathbb S^{d-1}_+$ and that its restriction to $\mathbb S^{d-1}_+$ coincides with $l$. Thus Equation \eqref{eq:conv_sets_A_x_beta_proof} can be rewritten as
	\[
	\P \big(|\boldX| / t > u, \, \boldX / |\boldX| \in \cdot \mid |\boldX| > t \big)
	\stackrel{d}{\to} u^{-\alpha} l( \cdot )\, , \quad t \to \infty\, , \quad u > 1\, .
	\]
	and therefore proves that $\boldX$ is regularly varying with tail index $\alpha$ and spectral measure $l(\cdot)$.
\end{proof}

Our aim is now to prove that the condition \ref{lem:equiv_2} holds true for a sparsely regularly varying random vector which satisfies assumption (A).

\paragraph{Proof of Theorem \ref{theo:relation_G_Theta}}

We consider a random vector $\boldX \in \SRV(\alpha, \boldZ)$ which satisfies assumption (A). For $\beta \in {\cal P}_d^*$ and $\boldx \in {\cal X}_\beta$ we decompose the probability $\P\big( |\boldX| / t \leq 1+\epsilon, \, \boldX / |\boldX| \in A_\boldx \mid |\boldX| > t \big)$ as follows:
\[
\sum_{\gamma \supset \beta}
\P\big( |\boldX| / t \leq 1+\epsilon, \, \boldX / |\boldX| \in A_\boldx,\, \pi(\boldX/t) \in C_\gamma \mid |\boldX| > t \big)\, ,
\]
where the restriction of the sum to the directions $\gamma \supset \beta$ is a consequence of the first point of Lemma \ref{lem:upper_lower_bound}. Moreover, for $\gamma \supset \beta$ Lemma \ref{lem:upper_lower_bound} then ensures that the probability $\P\big( |\boldX| / t \leq 1+\epsilon, \, \boldX / |\boldX| \in A_\boldx,\, \pi(\boldX/t) \in C_\gamma \mid |\boldX| > t \big)$ is upper-bounded by
\begin{equation}\label{eq:upper_bound}
\P\big( |\boldX| / t \leq 1+\epsilon, \, \pi(\boldX/t) \in A_{\boldx - \epsilon / |\gamma|} \cap C_\gamma \mid |\boldX| > t \big)\, ,
\end{equation}
assuming that $\epsilon > 0$ is small enough so that the vector $\boldx - \epsilon / |\gamma|$ has positive coordinates, and lower-bounded by
\begin{equation}\label{eq:lower_bound}
\P\big( |\boldX| / t \leq 1+\epsilon, \, \pi(\boldX/t) \in A_{\boldx (1+\epsilon)} \cap C_\gamma \mid |\boldX| > t \big)\, .
\end{equation}
In order to deal simultaneously with both probabilities in \eqref{eq:upper_bound} and \eqref{eq:lower_bound} we write
\begin{equation}\label{eq:upper_lower_bound}
\P\big( |\boldX| / t \leq 1+\epsilon, \, \pi(\boldX/t) \in A_{\psi_\epsilon(\boldx)} \cap C_\gamma \mid |\boldX| > t \big)\, ,
\end{equation}
where $\psi_\epsilon(\boldx)$ is either equal to $\boldx - \epsilon / |\gamma|$ or to $\boldx (1+\epsilon)$.

The proof is then divided into two steps. The first step consists in proving that these two probabilities \eqref{eq:upper_bound} and \eqref{eq:lower_bound} converge when $t\to \infty$ for $\lambda_\beta$-almost every $\boldx \in {\cal X}_\beta$. Then, the goal is to prove that after a division by $\epsilon$ these two limits converge to the same quantity when $\epsilon \to 0$. We extract the first step as a separate lemma.

\begin{lem}\label{lem:conv_t_main_theorem}
	For all $\beta \subset \gamma \in {\cal P}_d^*$, for $\lambda_\beta$-almost every $\boldx \in {\cal X}_\beta$, for almost every $\epsilon > 0$ we have the following convergence:
	\begin{equation}\label{eq:conv_pi_A_x}
	\lim \limits_{t \to \infty}
	\P\big( \pi(\boldX/t) \in A_{\psi_\epsilon(\boldx)} \cap C_\gamma, \, |\boldX| / t > 1+\epsilon \mid |\boldX| > t \big)
	= \P( \boldZ \in A_{\psi_\epsilon(\boldx)}  \cap C_\gamma, \, Y > 1+\epsilon) \, .
	\end{equation}
\end{lem}

\begin{proof}
We use Proposition \ref{prop:dependence_Z_Y} which entails that the right-hand side in \eqref{eq:conv_pi_A_x} is equal to
\begin{align*}
\P( \boldZ \in A_{\psi_\epsilon(\boldx)}  \cap C_\gamma, \, Y > 1+\epsilon)
&= \P( \boldZ \in A_{\psi_\epsilon(\boldx)}  \cap C_\gamma \mid Y > 1+\epsilon) \P(Y > 1+\epsilon)\\
&= \P( \pi((1+\epsilon) \boldZ) \in A_{\psi_\epsilon(\boldx)}  \cap C_\gamma) (1+\epsilon)^{-\alpha}\, .
\end{align*}
Regarding the left-hand side in \eqref{eq:conv_pi_A_x}, it can be rewritten as follows:
\begin{align*}
\P\big( \pi(\boldX/t) \in A_{\psi_\epsilon(\boldx)}  \cap C_\gamma, \, |\boldX| / t >  1&+\epsilon \mid |\boldX| > t \big)\\
&= \P\big( \pi(\boldX/t) \in A_{\psi_\epsilon(\boldx)}  \cap C_\gamma \mid  |\boldX| > t (1+\epsilon) \big) \frac{\P(|\boldX| > t (1+\epsilon) )}{\P(|\boldX| > t)}\, .
\end{align*}
The ratio $\P(|\boldX| > t (1+\epsilon)) / \P(|\boldX| > t)$ converges to $ (1+\epsilon)^{-\alpha}$ when $t \to \infty$ since $|\boldX|$ is regularly varying with tail index $\alpha$. Besides the probability $\P\big( \pi(\boldX/t) \in A_{\psi_\epsilon(\boldx)}  \cap C_\gamma \mid  |\boldX| > t  (1+\epsilon) \big)$ converges when $t \to \infty$ if and only if $\P\big( \pi( (1+\epsilon) \boldX/t) \in A_{\psi_\epsilon(\boldx)} \cap C_\gamma \mid  |\boldX| > t \big)$ converges when $t \to \infty$, and then both probabilities have the same limit. We use Lemma \ref{lem:iteration_projection} and the relation $\pi_z(\boldv) = z \pi(\boldv/z)$ for $\boldv \in \R^d_+$ and $z > 0$ which entail that
\[
\pi( (1+\epsilon) \boldX/t)
=  (1+\epsilon) \pi_{1/ (1+\epsilon)} (\boldX/t)
=  (1+\epsilon) \pi_{1/ (1+\epsilon)} ( \pi(\boldX/t) )
= \pi (  (1+\epsilon)\pi(\boldX/t) )\, .
\]
Hence the convergence in Equation \eqref{eq:conv_pi_A_x} holds if and only if
\begin{equation}\label{eq:conv_pi_A_x_temp}
\P\big( \pi (  (1+\epsilon) \pi(\boldX/t) ) \in A_{\psi_\epsilon(\boldx)}  \cap C_\gamma \mid |\boldX| > t \big)
\to \P( \pi( (1+\epsilon)\boldZ) \in A_{\psi_\epsilon(\boldx)} \cap C_\gamma)\, ,
\quad t \to \infty \, .
\end{equation}
The equivalence \eqref{eq:equiv_pi_A_x_beta} then entails that the former convergence holds if and only if the probability
\begin{align*}
\P\bigg(
\phi_\gamma(\pi(\boldX/t))_\beta \geq \frac{\psi_\epsilon(\boldx)}{1+\epsilon} + \frac{\epsilon}{(1+\epsilon) |\gamma|}, \,
\min_{i \in \gamma \setminus \beta} \phi_\gamma(\pi(\boldX\bigg.&/t))_i > \frac{\epsilon}{(1+\epsilon) |\gamma|}, \, \\
&\left.\max_{j \in \gamma^c} \phi_\gamma(\pi(\boldX/t))_j \leq \frac{\epsilon}{(1+\epsilon) (|\gamma|+1)}
\; \middle| \; |\boldX| > t
\right)
\end{align*}
converges to
\[
\P\bigg(
\phi_\gamma(\boldZ)_\beta \geq \frac{\psi_\epsilon(\boldx)}{1+\epsilon} + \frac{\epsilon}{(1+\epsilon)|\gamma|}, \,
\min_{i \in \gamma \setminus \beta} \phi_\gamma(\boldZ)_i > \frac{\epsilon}{(1+\epsilon) |\gamma|}, \, 
\max_{j \in \gamma^c} \phi_\gamma(\boldZ)_j \leq \frac{\epsilon}{(1+\epsilon) (|\gamma|+1)}
\bigg)\, ,
\]
when $t \to \infty$.

For $\gamma \supset \beta$ the set of all $\epsilon > 0$ for which there exists $\P( \min_{i \in \gamma \setminus \beta} \phi_\gamma(\boldZ)_i = \epsilon / [(1+\epsilon) |\gamma|]) > 0$ or for which $\P( \max_{j \in \gamma^c} \phi_\gamma(\boldZ)_j = \epsilon / [(1+\epsilon) (|\gamma|+1)]) > 0$ is at most countable so that for almost every $\epsilon > 0$ we have
\[
\P\Big(
\min_{j \in \gamma \setminus \beta} \phi_\gamma(\boldZ)_j = \frac{\epsilon}{(1+\epsilon) |\gamma|}
\Big)
=
\P\Big(
\max_{i \in \gamma^c} \phi_\gamma(\boldZ)_i = \frac{\epsilon}{(1+\epsilon) (|\gamma|+1)}
\Big)
=
0\, ,  \text{ for all } \gamma \supset \beta\, .
\]
Let us denote by ${\cal E}_\beta$ the set of all these $\epsilon >0$. The same argument implies that for $\lambda_\beta$-almost every $\boldx \in {\cal X}_\beta$ we have
\[
\P\Big(
\phi_\gamma(\boldZ)_j = \frac{\psi_\epsilon(\boldx)_j}{1+\epsilon} + \frac{\epsilon}{(1+\epsilon)|\gamma|} \Big) = 0, \,
\quad \text{for all } j \in \beta\, , \quad \gamma \supset \beta, \,  \text{ and } \epsilon \in {\cal E}_\beta\, .
\]
Let us denote by $\bar{\cal X}_\beta$ the set of all these $\boldx$. We have proved that for all $\epsilon \in {\cal E}_\beta$ and $\boldx \in \bar{\cal X}_\beta$ we have the convergence
\[
\lim \limits_{t \to \infty}
\P\big( \pi(\boldX/t) \in A_{\psi_\epsilon(\boldx)} \cap C_\gamma, \, |\boldX| / t > 1+\epsilon \mid |\boldX| > t \big)
= \P( \boldZ \in A_{\psi_\epsilon(\boldx)}  \cap C_\gamma, \, Y > 1+\epsilon) \, .
\]
This concludes the proof of the Lemma \ref{lem:conv_t_main_theorem}.
\end{proof}

\paragraph{End of the proof of Theorem \ref{theo:relation_G_Theta}}
We use the same notation $\bar {\cal X}_\beta$ and ${\cal E}_\beta$ as in the proof of Lemma \ref{lem:conv_t_main_theorem}. We start with the following decomposition:
\begin{align*}
\P\big( \pi ( \boldX/t ) \in A_{\psi_\epsilon(\boldx)}  \cap C_\gamma, \, |\boldX|/t \leq 1+\epsilon \mid  |&\boldX| > t \big)
=
\P\big( \pi ( \boldX/t ) \in A_{\psi_\epsilon(\boldx)}  \cap C_\gamma, \, |\boldX|/t \leq 1+\epsilon^2 \mid |\boldX| > t \big)\\
&+
\P\big( \pi ( \boldX/t ) \in A_{\psi_\epsilon(\boldx)}  \cap C_\gamma, \, |\boldX|/t  \in (1+\epsilon^2, 1+\epsilon] \mid |\boldX| > t \big)
\, .
\end{align*}
The introduction of the  term with $\epsilon^2$ on the right-hand side follows from the fact that the probability $\P\big( \pi ( \boldX/t ) \in A_{\psi_\epsilon(\boldx)}  \cap C_\gamma \mid |\boldX| > t \big)$ may not converge. The latter probability is bounded by $\P\big( |\boldX|/t \leq 1+\epsilon^2 \mid |\boldX| > t \big) \to 1 - (1+\epsilon^2)^{-\alpha} = o(\epsilon)$. Moreover, Lemma \ref{lem:conv_t_main_theorem} ensures that
\[
\P\big( \pi ( \boldX/t ) \in A_{\psi_\epsilon(\boldx)}  \cap C_\gamma, \, |\boldX|/t \in (1+\epsilon^2, 1+\epsilon] \mid |\boldX| > t \big)
\to
\P\big( \boldZ \in A_{\psi_\epsilon(\boldx)}  \cap C_\gamma, \, Y \in (1+\epsilon^2, 1+\epsilon] \big)\, ,
\]
when $t \to \infty$ for $\epsilon, \epsilon^2 \in {\cal E}_\beta$ and $\boldx \in \bar{\cal X}_\beta$. Hence we obtain
\begin{align*}
\lim_{t \to \infty} \P\big( \pi ( \boldX/t ) \in A_{\psi_\epsilon(\boldx)}  \cap C_\gamma&, \, |\boldX|/t \leq 1+\epsilon \mid  |\boldX| > t \big)\\
&=
\lim_{t \to \infty}
\P\big( \pi ( \boldX/t ) \in A_{\psi_\epsilon(\boldx)}  \cap C_\gamma, \, |\boldX|/t  \in (1+\epsilon^2, 1+\epsilon] \mid |\boldX| > t \big) + o(\epsilon)\\
&=
\P\big( \boldZ \in A_{\psi_\epsilon(\boldx)}  \cap C_\gamma, \, Y  \in (1+\epsilon^2, 1+\epsilon] \big) + o(\epsilon)\\
&=
\P\big( \boldZ \in A_{\psi_\epsilon(\boldx)}  \cap C_\gamma, \, Y  \leq 1+\epsilon \big) + o(\epsilon)
\, .
\end{align*}
Then, Proposition \ref{prop:dependence_Z_Y} implies that
\begin{align}
\P\big( \boldZ_\beta \in A_{\psi_\epsilon(\boldx)} \cap C_\gamma&, \, Y \leq 1+\epsilon \big) \nonumber\\
&= \P\big( \boldZ \in A_{\psi_\epsilon(\boldx)} \cap C_\gamma\big)
- \P\big( \boldZ\in A_{\psi_\epsilon(\boldx)} \cap C_\gamma \mid Y > 1+\epsilon \big) (1+\epsilon)^{-\alpha} \nonumber\\
&= \P\big( \boldZ \in A_{\psi_\epsilon(\boldx)} \cap C_\gamma \big)
- \P\big( \pi((1+\epsilon)\boldZ) \in A_{\psi_\epsilon(\boldx)} \cap C_\gamma \big) (1+\epsilon)^{-\alpha} \nonumber\\
&= \big[1 - (1+\epsilon)^{-\alpha}\big]
\P\big( \boldZ \in A_{\psi_\epsilon(\boldx)} \cap C_\gamma \big) \label{eq:proba_proof_first_term}\\
& \hspace{1cm} + (1+\epsilon)^{-\alpha} \Big[
\P\big( \boldZ \in A_{\psi_\epsilon(\boldx)} \cap C_\gamma \big)
-
\P\big( \pi((1+\epsilon)\boldZ)\in A_{\psi_\epsilon(\boldx)} \cap C_\gamma \big) \Big]\label{eq:proba_proof_second_term}\, .
\end{align}
The sets $A_{\psi_\epsilon(\boldx)}$ decrease when $\epsilon \to 0$ and satisfy $\cap_{\epsilon >0} A_{\psi_\epsilon(\boldx)} = A_\boldx$. Hence the term in \eqref{eq:proba_proof_first_term} divided by $\epsilon$ converges to $\alpha \P( \boldZ \in A_\boldx \cap C_\gamma)$. For the term in \eqref{eq:proba_proof_second_term} we use \eqref{eq:equiv_pi_A_x_beta} which states that the event $\{ \pi((1+\epsilon)\boldZ) \in A_{\psi_\epsilon(\boldx)} \cap C_\gamma \}$ corresponds to
\[
\left\{
\begin{array}{lll}
&\phi_\gamma(\boldZ)_j \geq \frac{\psi_\epsilon(\boldx)_j}{(1+\epsilon)} + \frac{\epsilon}{|\gamma| (1+\epsilon)}\, , \quad j \in \beta \, ,\\
&\min_{j \in \gamma \setminus \beta}  \phi_\gamma(\boldZ)_j > \frac{\epsilon}{|\gamma|(1+\epsilon)}\\
&\max_{j \in \gamma^c} \phi_\gamma(\boldZ)_j  \leq \frac{\epsilon}{(|\gamma|+1)(1+\epsilon)}\, .
\end{array}
\right.
\]
Hence the difference of probabilities in \eqref{eq:proba_proof_second_term} corresponds to the difference
\[
H_{\beta, \gamma} \big( \psi_\epsilon(\boldx), 0, 0 \big)
-
H_{\beta, \gamma} \Big(
\frac{\psi_\epsilon(\boldx)}{(1+\epsilon)} + \frac{\epsilon}{|\gamma| (1+\epsilon)},
\frac{\epsilon}{|\gamma|(1+\epsilon)},
\frac{\epsilon}{(|\gamma|+1)(1+\epsilon)}
\Big)\, .
\]
After a division by $\epsilon$ this difference converges to
\[
\mathrm d H_{\beta, \gamma} (\boldx_\beta, 0, 0)
\cdot
(\boldx_\beta -1/|\gamma|, -1/|\gamma| , -1/(|\gamma|+1))\, ,
\quad \epsilon \to 0\, ,
\]
for $\psi_\epsilon(\boldx) = \boldx/(1+\epsilon)$ and $\psi_\epsilon(\boldx) = \boldx - \epsilon/|\gamma|$.

All in all, we proved that for $\lambda_\beta$-almost every $\boldx \in {\cal X}_\beta$ both limits
\[
\lim_{\epsilon \to 0}
\liminf_{t \to \infty}
\epsilon^{-1} \P\big( |\boldX| / t \leq 1+\epsilon, \, \boldX / |\boldX| \in A_\boldx \mid |\boldX| > t \big)
\]
and
\[
\lim_{\epsilon \to 0}
\limsup_{t \to \infty}
\epsilon^{-1} \P\big( |\boldX| / t \leq 1+\epsilon, \, \boldX / |\boldX| \in A_\boldx \mid |\boldX| > t \big)
\]
exist and are equal to
\begin{align*}
l(A_\boldx)
&=
\sum_{\gamma \supset \beta}
\alpha \P(\boldZ \in  A_\boldx  \cap C_\gamma)
+ \sum_{\gamma \supset \beta} \mathrm d H_{\beta, \gamma} (\boldx_\beta, 0, 0) \cdot \Big(\boldx_\beta - \frac{1}{|\gamma|}, -\frac{1}{|\gamma|}, -\frac{1}{|\gamma|+1} \Big)\\
&= \alpha \P(\boldZ \in  A_\boldx)
+ \sum_{\gamma \supset \beta} \mathrm d H_{\beta, \gamma} (\boldx_\beta, 0, 0) \cdot \Big(\boldx_\beta - \frac{1}{|\gamma|}, -\frac{1}{|\gamma|}, -\frac{1}{|\gamma|+1} \Big)\, ,
\end{align*}
Then assumption (A) ensures that for $\lambda_\beta$-almost every $\boldx \in {\cal X}_\beta$ the function $\boldx \mapsto l(A_\boldx)$ is continuous at $\boldx$ and Lemma \ref{lem:rv_equivalent_condition} allows us to conclude.

\subsection{Proof of Proposition \ref{prop:Z_on_specific_subsets}}
Recall that we have defined the quantities $\theta_{\beta, \, i} = \sum_{j \in \beta} (\theta_j - \theta_i)$ and $\theta_{\beta, \, i, \, +} = \sum_{j \in \beta} (\theta_j - \theta_i)_+$ for $\boldtheta \in \mathbb S^{d-1}_+$, $\beta \in {\cal P}_d^*$, and $1 \leq i \leq d$.\\

\noindent 1. We only prove the convergence \eqref{eq:cv_projection_Z_on_C_beta} (the proof of \eqref{eq:cv_projection_Z_on_beta} is similar). For $\beta \in {\cal P}_d^*$, Lemma \ref{lem:proj_beta_betac} ensures that $\pi(Y \boldTheta)\in C_\beta$ if and only if $\max_{i \in \beta} \Theta_{\beta, \, i}  < 1/Y$ and $\min_{i \in \beta^c} \Theta_{\beta, \, i} \geq 1/Y$. Hence \eqref{eq:cv_projection_Z_on_C_beta} is equivalent to
\begin{equation} \label{eq:conv_Y_Theta_D_beta}
\P\big( (|\boldX|/t, \boldX/|\boldX|) \in D_\beta \mid |\boldX|>t \big) \to \P((Y,\boldTheta) \in D_\beta)\, ,
\end{equation}
with
\[
D_\beta = \big\{ (r, \theta) \in (1,\infty) \times \mathbb S^{d-1}_+ : \theta_{\beta, \, i} < 1/r \text{ for } i \in \beta, \text{ and } \theta_{\beta, \, i} \geq 1/r \text{ for } i \in \beta^c \big\}\, ,
\]
and this convergence holds if $\P((Y,\boldTheta) \in \partial D_\beta) = 0$. The boundary $\partial D_\beta$ satisfies the inequality
\begin{align*}
\P((Y,\boldTheta) \in \partial D_\beta)
\leq \sum_{i = 1}^d \P ( \Theta_{\beta, \, i} = Y^{-1} )\, ,
\end{align*}
and all the terms of the sum are null since $Y$ is a continuous random variable independent of $\boldTheta$. Thus $\P((Y,\boldTheta) \in \partial D_\beta) = 0$, which implies that convergence \eqref{eq:conv_Y_Theta_D_beta} holds and then that convergence \eqref{eq:cv_projection_Z_on_C_beta} holds as well.\\

\noindent 2. Following Lemma \ref{lem:proj_beta_betac}, the probability that $\boldZ$ belongs to $C_\beta$ is equal to
\begin{align*}
\P(\boldZ \in C_\beta)
&= \P \Big( \max_{j \in \beta} \sum_{k \in \beta} (Y \Theta_k - Y \Theta_j) < 1, \, \min_{j \in \beta^c} \sum_{k\in \beta} (Y \Theta_k - Y \Theta_j) \geq 1 \Big)\\
&= \P \Big( \max_{j \in \beta} \, \Theta_{\beta, \, j, \, +}^\alpha < Y^{-\alpha}, \, \min_{j \in \beta^c} \, \Theta_{\beta, \, j, \, +}^\alpha \geq Y^{-\alpha} \Big)\\
&= \int_0^1 \P \Big( \max_{j \in \beta} \Theta_{\beta, \, j, \, +}^\alpha < u \leq \min_{j \in \beta^c} \Theta_{\beta, \, j, \, +}^\alpha \Big) \; \mathrm d u\\
&= \E \Big[ \Big( \min_{j \in \beta^c} \Theta_{\beta, \, j, \, +}^\alpha - \max_{j \in \beta} \Theta_{\beta, \, j, \, +}^\alpha \Big)_+ \Big]\, .
\end{align*}

Similarly, Equation \eqref{eq:Z_beta_null_components} is also a consequence of Lemma \ref{lem:proj_beta_betac} since we have the relations
\begin{align}
\P(\boldZ_{\beta^c} = 0)
&= \P\Big( 1 \leq \min_{j \in \beta^c} \sum_{k=1}^d (Y\Theta_k - Y\Theta_j)_+\Big)
= \P\big(Y^{-\alpha} \leq \min_{j \in \beta^c} \Theta_{\beta, \, j, \, +}^\alpha \big) \label{eq:Z_beta_c_zero}\\
&= \int_0^1 \P\big(u \leq \min_{j \in \beta^c} \Theta_{\beta, \, j, \, +}^\alpha \big) \, \mathrm d u
= \E \big[ \min_{j \in \beta^c} \Theta_{\beta, \, j, \, +}^\alpha \big] \, . \nonumber
\end{align}
This proves \eqref{eq:Z_beta_null_components} and concludes the proof of the proposition.

\subsection{Proof of Theorem \ref{theo:comparison_Theta_Z_maximal_subsets}}

We first state an inequality which will be used to prove both results of Theorem \ref{theo:comparison_Theta_Z_maximal_subsets}.
	
\begin{lem}\label{lem:condition_subsets_C_beta}
	For $\beta \in {\cal P}_d^*$ we have the inequality
	\begin{equation} \label{eq:inequality_on_C_beta}
	\P(\boldTheta \in C_\beta) \leq \P \big( \max_{j \in \beta} \Theta_{\beta, \, j, \, +} < 1 \big) \, .
	\end{equation}
\end{lem}
	
\begin{proof}[Proof of Lemma \ref{lem:condition_subsets_C_beta}]
	The relation in \eqref{eq:inequality_on_C_beta} is equivalent to
	\[
	\P \big( \max_{j \in \beta} \Theta_{\beta, \, j, \, +} = 1 \big) \leq \P(\boldTheta \notin C_\beta)\, .
	\]
	The probability on the left-hand side can be rewritten as
	\[
	\P \big( \max_{j \in \beta} \Theta_{\beta, \, j, \, +} = 1 \big)
	= \P \Big( \sum_{k \in \beta} (\Theta_k - \min_{j \in \beta}  \Theta_j) = 1 \Big)
	= \P \Big( \sum_{k \in \beta} \Theta_k = 1 + |\beta| \min_{j \in \beta}  \Theta_j \Big)\, .
	\]
	Since $\boldTheta \in \mathbb S^{d-1}_+$, the equality $\sum_{k \in \beta} \Theta_k = 1 + |\beta| \min_{j \in \beta}  \Theta_j$ holds only if there exists $k \in \beta$ such that $\Theta_k = 0$. Thus we obtain the inequality
	\[
	\P \Big( \max_{j \in \beta} \Theta_{\beta, \, j, \, +} = 1 \Big)
	\leq \P( \Theta_k = 0 \text{ for some } k \in \beta)
	\leq \P( \boldTheta \notin C_\beta )\, ,
	\]
	which concludes the proof.
\end{proof}
	
We now move on to the proof of Theorem \ref{theo:comparison_Theta_Z_maximal_subsets}. The proof of these results relies on Equation \eqref{eq:Z_C_beta} and Lemma \ref{lem:condition_subsets_C_beta} above.

\noindent 1. If $\beta \in {\cal P}_d^*$ is such that $\P(\boldTheta \in C_\beta) > 0$, then Equation \eqref{eq:Z_C_beta} entails that
\begin{align}
\P( \boldZ \in C_\beta )
& \geq \E \Big[ \Big(
\min_{j \in \beta^c} \Theta_{\beta, \, j, \, +}^\alpha - \max_{j \in \beta} \Theta_{\beta, \, j, \, +}^\alpha
\Big)_+ \indic_{ \boldTheta \in C_\beta }\Big] \nonumber\\
& = \E \Big[ \Big(
1 - \max_{j \in \beta} \Theta_{\beta, \, j, \, +}^\alpha
\Big) \indic_{ \boldTheta \in C_\beta } \Big] \nonumber\\
& = \E \Big[ \Big(
1 - \max_{j \in \beta} \Theta_{\beta, \, j, \, +}^\alpha
\Big) \mid \boldTheta \in C_\beta\Big] \P(\boldTheta \in C_\beta)\, . \label{eq:proof_relation_Theta_Z}
\end{align}
The expectation is positive by Lemma \ref{lem:condition_subsets_C_beta} and the probability $\P(\boldTheta \in C_\beta)$ is positive by assumption. This shows that $\P(\boldZ \in C_\beta) > 0$.\\

\noindent 2. We separately prove both implications.

- We first consider a maximal direction $\beta$ of $\boldTheta$. The first point of the theorem ensures that $\P(\boldZ \in C_\beta) > 0$. Besides, if $\beta' \supsetneq \beta$, then Equation \eqref{eq:inequality_beta} gives
\[
\P(\boldZ \in C_{\beta'}) \leq \P(\boldZ_{\beta'} > 0) \leq \P(\boldTheta_{\beta'} > 0)\, .
\]
and this last probability equals zero since $\beta$ is a maximal direction for $\boldTheta$. This proves that $\beta$ is a maximal direction for $\boldZ$.

- We now consider a maximal direction $\beta$ of $\boldZ$. We claim that for $\beta' \supsetneq \beta$, $\P(\boldTheta \in C_{\beta'}) = 0$. If not, the first point of the theorem entails that $\P(\boldZ \in C_{\beta'}) > 0$ which contradicts the maximality of $\beta$ for $\boldZ$.

Secondly, Equation \eqref{eq:Z_C_beta} gives that
\begin{align} \label{eq:proof_maximality_subsets}
\P(\boldZ \in C_\beta)
&= \E \Big[ \Big(
\min_{j \in \beta^c} \Theta_{\beta, \, j, \, +}^\alpha
- \max_{j \in \beta} \Theta_{\beta, \, j, \, +}^\alpha
\Big)_+ \Big] \nonumber\\
&= \E \Big[ \Big(
\min_{j \in \beta^c} \Theta_{\beta, \, j, \, +}^\alpha
- \max_{j \in \beta} \Theta_{\beta, \, j, \, +}^\alpha
\Big)_+ \indic_{\boldTheta \in C_\beta} \Big]
+ \E \Big[ \Big(
\min_{j \in \beta^c} \Theta_{\beta, \, j, \, +}^\alpha
- \max_{j \in \beta} \Theta_{\beta, \, j, \, +}^\alpha
\big)_+ \indic_{\boldTheta \notin C_\beta} \Big] \nonumber\\
&= E_1 + E_2\, .
\end{align}
The first term $E_1$ has already been calculated in \eqref{eq:proof_relation_Theta_Z}. It is equal to
\[
E_1 = \E \Big[
1 - \max_{j \in \beta} \Theta_{\beta, \, j, \, +}^\alpha \mid \boldTheta \in C_\beta
\Big] \P( \boldTheta \in C_\beta)\, .
\]
For the second term $E_2$, the assumption $\boldTheta \notin C_\beta$ implies that there exists $l \in \beta$ such that $\Theta_l = 0$, or that there exists $r \in \beta^c$ such that $\Theta_r >0$. We then decompose $E_2$ into two terms:
\begin{align*}
&\E \Big[ \Big(
\min_{j \in \beta^c} \Theta_{\beta, \, j, \, +}^\alpha - \max_{j \in \beta} \Theta_{\beta, \, j, \, +}^\alpha
\Big)_+ \indic_{\boldTheta \notin C_\beta} \Big] \\
&\leq \E \Big[ \Big(
\min_{j \in \beta^c} \Theta_{\beta, \, j, \, +}^\alpha - \max_{j \in \beta} \Theta_{\beta, \, j, \, +}^\alpha
\Big)_+ \indic_{\exists l \in \beta, \, \Theta_l =0} \Big] \\
&+ \E\Big[ \Big(
\min_{j \in \beta^c} \Theta_{\beta, \, j, \, +}^\alpha - \max_{j \in \beta} \Theta_{\beta, \, j, \, +}^\alpha
\Big)_+ \indic_{ \exists \beta' \supsetneq\beta,\, \boldTheta \in C_{\beta'}} \Big]\, .
\end{align*}
The first expectation is then equal to
\[\E \Big[ \Big(
\min_{j \in \beta^c} \Theta_{\beta, \, j, \, +}^\alpha - \Big(\sum_{k \in \beta} (\Theta_k)_+\Big)^\alpha
\Big)_+ \indic_{\exists l \in \beta, \, \Theta_l =0} \Big]\, ,
\]
and thus vanishes. The second expectation is smaller than $\P(\boldTheta \in C_{\beta'} \text{ for some } \beta' \supsetneq\beta)$ which is equal to zero. Indeed, if $\P(\boldTheta \in C_{\beta'}  \text{ for some } \beta' \supsetneq\beta) > 0$, then by Equation \eqref{eq:Z_C_beta}, we also have $\P(\boldZ \in C_{\beta'}  \text{ for some } \beta' \supsetneq\beta) >0$ which contradicts the maximality of $\beta$ for $\boldZ$. All in all this proves that $E_2 = 0$.
	
Going back to Equation \eqref{eq:proof_maximality_subsets}, we have proved that
\[
\P(\boldZ \in C_\beta)
= E_1
= \E\Big[
1 - \max_{j \in \beta} \Theta_{\beta, \, j, \, +}^\alpha \mid \boldTheta \in C_\beta
\Big] \P ( \boldTheta \in C_\beta )\, .
\]
By Lemma \ref{lem:condition_subsets_C_beta} we know that the expectation is positive. Hence, the assumption $\P(\boldZ \in C_\beta) > 0$ implies that $\P(\boldTheta \in C_\beta) > 0$ which proves that $\beta$ is a maximal direction of $\boldTheta$.

\paragraph{Acknowledgments}
We are grateful to two referees for careful reading of the paper and for useful suggestions. We also would like to thank Johan Segers for insightful comments which led to complete the proof of Theorem \ref{theo:relation_G_Theta}.

\small{
\bibliographystyle{apalike} 
\bibliography{biblio.bib}} 

\end{document}